\documentclass[prodmode,acmtecs]{acmsmall} 
\usepackage{url}
\usepackage{amsmath, bbm, amsthm, amsfonts, amssymb, cite, enumerate, ctable}
\usepackage{cancel}
\usepackage{datetime}

\usepackage[ruled]{algorithm2e}

\SetAlFnt{\small}
\SetAlCapFnt{\small}
\SetAlCapNameFnt{\small}
\SetAlCapHSkip{0pt}
\IncMargin{-\parindent}

\acmVolume{9}
\acmNumber{4}
\acmArticle{39}
\acmYear{2010}
\acmMonth{3}

\DeclareMathOperator*{\im}{im}
\DeclareMathOperator*{\len}{len}
\DeclareMathOperator*{\unif}{unif}
\DeclareMathOperator*{\argmin}{arginf}

\DeclareMathOperator*{\expec}{\mathbf E}

\newcommand*\loss{\ensuremath{\boldsymbol\ell}}

\newcommand{\prob}{{\mathbb P}}
\newcommand{\probq}{{\mathbb Q}}

\newcommand{\iid}{\mathsf{SLT}}
\newcommand{\seq}{\mathsf{SEQ}}
\newcommand{\uni}{\mathsf{UNI}}
\newcommand{\erm}{\mathsf{ERM}}
\newcommand{\sol}{\mathsf{SOL}}

\newcommand{\risk}{{\mathbf R}}

\newcommand{\error}{{\mathbf V}}

\newcommand{\kolmog}{{\mathbf K}}
\newcommand{\kl}{{\mathbf D}}

\newcommand{\turing}{{\mathcal T}}
\newcommand{\theory}{{\mathcal O}}
\newcommand{\ttheory}{\tilde{\mathcal O}}

\newcommand{\hypotheses}{{\mathcal H}}

\newcommand{\radem}{{\mathsf{Radem}}}
\newcommand{\vc}{{\mathsf{vc}}}
\newcommand{\ldim}{{\mathsf{ldim}}}
\newcommand{\information}{\mathsf{Gain}}
\newcommand{\cover}{\mathsf{Cover}}

\newcommand{\falsewt}{\mathbf{F}}
\newcommand{\falsect}{\mathbf{G}}

\newcommand{\bx}{{\mathbf x}}

\newcommand{\bz}{{\mathbf{z}}}
\newcommand{\X}{{\mathcal{X}}}
\newcommand{\Y}{{\mathcal{Y}}}

\newcommand{\bR}{{\mathbb R}}

\newcommand{\indicator}{{\mathbf I}}
\newcommand{\unint}{{\mathbb I}}

\newcommand{\bv}{{\mathbf v}}

\newcommand{\fm}{{\mathfrak m}}
\newcommand{\cS}{{\mathcal S}}

\newcommand{\bX}{{\mathbf X}}

\newtheorem{thm}{Theorem}

\newtheorem{prop}[thm]{Proposition}
\newtheorem{def_prop}[thm]{Definition-Proposition}
\newtheorem{lem}[thm]{Lemma}
\newtheorem{defn}{Definition}
\newtheorem{rem}{Remark}
\newtheorem{eg}{Example}

\newtheorem*{defn_risk}{Definition A}
\newtheorem*{defn_predictive}{Definition B}
\newtheorem*{defn_falsify}{Definition C}
\newtheorem*{thm_main}{Theorem D}
\newtheorem*{thm_main_r}{Theorem D''}
\newtheorem*{cor_main}{Corollary D'}
\newtheorem*{thm_uni}{Theorem E}

\begin{document}

\markboth{D Balduzzi}{
Falsifiable $\implies$ Learnable
}
\title{
Falsifiable $\implies$ Learnable
}
\author{David Balduzzi
\affil{Victoria University of Wellington}}

\begin{abstract}
	The paper demonstrates that falsifiability is fundamental to learning. We prove the following theorem for statistical learning and sequential prediction: If a theory is falsifiable then it is learnable -- i.e. admits a strategy that predicts optimally. An analogous result is shown for universal induction. 

	\begin{quote}
		\emph{A theory that explains everything, [predicts] nothing.} -- attributed to Karl Popper.
	\end{quote}
\end{abstract}

\maketitle

\terms{Learning, Generalization, Semantics}
\acmformat{David Balduzzi, 2014. Falsifiable $\implies$ Learnable.}

\setcounter{section}{-1}

\section{Introduction}

To what extent are theory-based predictions justified by prior observations? The question is known as the problem of induction and is fundamental to scientific inference. We address the problem of induction from the perspective of learning theory. That is, we consider which theories, and under what assumptions, can be applied to make optimal predictions. 

Our main result is that the more hypotheses a theory falsifies, suitably quantified, the closer the predictive performance of the best strategy (based on the theory) will be to the theory's \emph{post hoc} explanatory performance on observed data.

\setcounter{subsection}{-1}
\subsection{Non-technical overview (or, Learning theory for the working scientist)}

Learning theorists have characterized the generalization performance of algorithms in a wide range of scenarios. Although none of these scenarios adequately captures the practice of scientific inference, they form a family of minimal models of prediction. 

An intuitive understanding of the main results of learning theory therefore belongs in every scientist's conceptual toolkit. Unfortunately, the results are phrased in opaque terminology that depends on specialized concepts such as Rademacher complexity, shattering coefficients and VC-dimensions. 

This paper presents basic results from learning theory in terminology that is meaningful to the broader scientific community.

The results cover three scenarios. In each scenario, Forecaster uses a theory (or theories) to predict Nature's next move(s) based on Nature's previous moves. 
\begin{enumerate}[S1.]
	\setcounter{enumi}{1}
	\item \emph{Statistical learning} ($\iid$). 
	Forecaster aims to predict events sampled from an unknown probability distribution based on a finite sample \cite{vapnik:95,boucheron:00,bousquet:04}.
	\item \emph{Sequential prediction} ($\seq$).
	Forecaster aims to predict events generated by an adversarial Nature that adapts to Forecaster's previous moves \cite{cesa:06,abernethy:09,rakhlin:14a}.
	\item \emph{Universal induction} ($\uni$).
	Forecaster aims to predict elements drawn from an arbitrarily chosen computable sequence \cite{solomonoff:64,hutter:11}.
\end{enumerate}

The paper develops the following account. 

\begin{enumerate}[A.]
	\item \emph{The risk.}
	\begin{itemize}
		\item The \textbf{risk of a theory} is how accurately it \emph{explains} a sequence of events.
	\end{itemize}
	A theory explains a sequence of events perfectly if it contains a predictor that correctly labels every element. In general, the accuracy of an explanation is the fraction of the sequence that its best predictor explains correctly.
	\begin{itemize}
		\item The \textbf{risk of a strategy} is how accurately it \emph{predicts} a sequence of events.
	\end{itemize}
	A strategy specifies picks a predictor based on previously observed events, which it then applies to future events. The strategy's predictive accuracy is the fraction of future events that it labels correctly.\\
	
	\item \emph{Learnability.}
	\begin{itemize}
		\item The \textbf{predictive risk} (or \textbf{regret}) on a sequence is the difference between a strategy's \emph{predictive} accuracy and the theory's \emph{explanatory} accuracy:
		\begin{equation*}
			\big\{\text{predictive risk}\big\} = \big\{\textrm{how well strategy predicts}\} - \big\{\textrm{how well theory explains}\big\}
		\end{equation*}
	\end{itemize}
	The predictive risk measures the strategy's effectiveness. It is not an absolute measure. Effectiveness is relative to a baseline -- how well the theory explains the sequence in hindsight. Thus, the predictive risk quantifies the cost from not knowing what Nature will do next, independently of the cost of not having a good model of Nature.
	\begin{itemize}
		\item A \textbf{strategy is optimal} if its predictive risk is asymptotically negligible on any sequence:
		\begin{equation*}
			\big\{\textrm{strategy optimal}\big\}
			\quad\textrm{if}\quad
			\left[\lim_{n\rightarrow \infty}\big\{\textrm{predictive risk}\big\}=0\right]
		\end{equation*}
	\end{itemize}
	The definition of optimal is subtle. An optimal strategy does not necessarily predict accurately. Rather, it predicts about as accurately as the theory explains. 
	\begin{itemize}
		\item A \textbf{theory is learnable} if it admits an optimal strategy:
		\begin{equation*}
			\big\{\textrm{theory learnable}\big\}
			\quad\textrm{if}\quad
			\exists\big\{\textrm{optimal strategy}\big\}
		\end{equation*}
	\end{itemize}
	In other words, a theory is learnable if it admits a strategy that predicts future events as well as the theory explains them after the fact.\\

	\item \emph{Falsifiability.}
	\begin{itemize}
		\item The \textbf{falsifiability of a theory} is the fraction of effective hypotheses about a sequence that it cannot explain. 
	\end{itemize}
	Effective hypotheses are hypotheses about finite sequences. The set of effective hypotheses is necessarily finite. We measure falsifiability in two ways, soft and hard:
	\begin{align*}
		\falsewt &
		:= 2\sum_{\epsilon \in \unint}  \Big(\text{\small{fraction of effective hypotheses falsified}}\Big) 
		\cdot \Big(\text{\small{on fraction $\epsilon$ of data}}\Big) \\
		\falsect & 
		:= \frac{\log\text{-\# of effective hypotheses that theory falsifies}}{\log \textrm{\# of effective hypotheses}}
	\end{align*}
	The two notions are, respectively, the expectation of a risk-induced distribution on errors and the risk's Bayesian information gain, see section~\ref{s:slt_false}. They are closely related to the statistical and sequential Rademacher complexities and covering numbers, and Kolmogorov complexity.
	\begin{itemize}
		\item A \textbf{theory is falsifiable} if the fraction of effective hypotheses that it falsifies tends to one asymptotically.
		\begin{equation*}
			\big\{\textrm{theory falsifiable}\big\}
			\quad\textrm{if}\quad
			\left[\lim_{n\rightarrow \infty}
			\big\{\textrm{falsifiability}\big\}=1
			\right]
		\end{equation*}
	\end{itemize}
	The number of effective hypotheses grows exponentially with sequence length, so the requirement is quite weak. For example, a theory is falsifiable if the number of hypotheses it explains grows polynomially.\\

	\item \emph{Falsifiable $\implies$ Learnable} ($\iid, \seq$).
	\begin{itemize}
		 \item \textbf{Main theorem (qualitative).}
	 	If a theory is falsifiable, then it is learnable:
	 	\begin{equation*}
	 		\big\{\textrm{falsifiable}\}\implies\big\{\textrm{learnable}\big\}
	 	\end{equation*}
	\end{itemize}
	Alternatively, if a theory is falsifiable then it admits a strategy that predicts optimally -- that is, a strategy that predicts any sequence as well, asymptotically, as the theory would have explained the sequence in hindsight.
	\begin{itemize}
		\item \textbf{Main theorem (quantitative).}
		\begin{equation*}
	 		\big\{\textrm{predictive risk}\big\} \leq 1 - \big\{\textrm{falsifiability}\big\}	
		\end{equation*}
	\end{itemize}
	The quantitative version of the main theorem provides guarantees -- across all sequences of some finite length $n$ -- on the expected performance of a theory's best strategy in terms of the falsifiability of the theory. The qualitative version is a corollary of the quantitative.\\

	\item \emph{Falsifiable $\implies$ Learnable} ($\uni$).\\
	Universal induction differs significantly from the other two scenarios. We reformulate Solomonoff induction to show that Forecaster constructs a nested sequence of theories in response to observations; from which predictors are drawn uniformly at random. Falsifiability is defined as above in this setting, but it admits a different interpretation:
	\begin{equation*}
		\big\{\text{falsifiability}\big\} = \big\{\log\textrm{-\# hypotheses Forecaster eliminates whilst adapting theory}\big\}	
	\end{equation*}
	Importantly, Forecaster eliminates hypotheses prior to -- and separately from -- making predictions.
	\begin{itemize}
	\item \textbf{Main theorem (quantitative).}
		\begin{equation*}
			\big\{\textrm{predictive risk}\big\} \leq \big\{\text{falsifiability}\big\}
		\end{equation*}		
	\end{itemize}
	In short, the number of hypotheses eliminated (or falsified) by Forecaster whilst adapting its theory controls its predictive performance.	
\end{enumerate}

\subsection{Outline of the paper and summary of the main contributions}

The paper is organized as follows. Section~\ref{s:represent} introduces two basic tools: the induced distribution and the Bayesian information gain. When a function has a finite domain, a natural prior on the domain is the uniform distribution, in which case the induced distribution and information gain can be interpreted as different ways of counting elements in pre-images.

The next three sections consider statistical learning, sequential prediction and universal induction in turn. The sections are variations on a basic template. 

The risk is the fundamental object in all three cases, Definition~A in sections~$x.1$ for $x=2,3,4$. The risk is a function from sequences of events to errors that can be computed with respect to strategies or theories. In the first case, the risk quantifies predictive performance of the strategy; in the second, it quantifies explanatory performance of the theory in hindsight. The predictive risk is the (minimax) difference between predictive and explanatory performance, Definition~B in sections~$x.2$.

An event is an ordered pair: a process acting on an input. The key step in the paper is to  reformulate the risk as a function from hypothetical processes to errors, by fixing the input sequence. The risk is then a function with a finite domain.

We propose two notions of falsifiability,\footnote{Only hard falsifiability is relevant to universal prediction.} Definition~C in sections~$x.3$. The first, soft falsifiability is the expected error under the risk-induced distribution on errors. Intuitively, it is a weighted sum of how many potential hypotheses are falsified over different fractions of the data. The second, hard falsifiability, is the risk's Bayesian information gain. Intuitively, it is the ``log-fraction'' of falsified hypotheses.

The main result is that soft and hard falsifiability control the predictive risk in all three scenarios, Theorems~D~\&~E in sections~$x.4$. Specifically, we show that falsifiability is equivalent to, or upper or lower bounds, the relevant measures of capacity: the statistical and sequential Rademacher complexities and covering numbers, and Kolmogorov complexity. The bounds on predictive risk then follow from standard results in learning theory \cite{boucheron:00,bousquet:04,hutter:11,rakhlin:14a}. Proofs are collected in sections~$x.5$.

The conclusion discusses the results' implications for Popper's account of scientific inference and the problem of induction, section~\ref{s:discussion}.

The main contributions are:
\begin{itemize}
	\item Relating the formal models of prediction developed by learning theorists to how working scientists think about scientific inference.
	\item Deriving falsifiability, and so the fundamental measures of capacity and complexity, as natural properties of the optimization problem at hand (the risk, Remark~\ref{r:opt}).
	\item Unifying basic notions from information theory, learning theory, and algorithmic complexity under the rubric of falsifiability.
\end{itemize}
The simplicity of the definitions and resulting theorems -- along with the fact that they apply across diverse settings -- suggest that falsifiability may be a more natural, flexible concept than capacity.

\subsection{Related work}

Connections between falsifiability and statistical learning theory were pointed out in \cite{vapnik:95,harman:07,corfield:09}. However, these works only considered VC dimension, which does not relate to falsifiability as directly as the measures introduced here. Moreover, they only considered the setting of statistical learning. 

Preliminary versions of this work were presented in \cite{balduzzi:11ilf,balduzzi:11ffp}.

\subsection{Notation}

We have endeavored to use similar notation for the three settings. Consequently, we have been forced to overload certain symbols. In particular, superscripts can refer to both Cartesian products, e.g. $X^n=\prod_{t=1}^n X$, and disjoint unions, e.g. $Y^\bullet=\bigcup_{n=1}^\infty Y^n$. 

\begin{center}
\begin{tabular*}{.8\textwidth}{|l|l||l|l|}
	\hline
	indicator function          & $\indicator$        	     &
	unit interval [0,1]         & $\unint$            	     \\
	0/1 loss                    & $\loss$             	     &
	set of distributions on $X$ & $\Delta_X$          	     \\
	expectation                 & $\expec$            	     &
	probability distribution    & $\prob$ or $\probq$ 	     \\
	risk                        & $\risk$             	     &
	Bayesian information gain   & $\information$      	     \\
	predictive risk (regret)    & $\error$            	     &
	Rademacher complexity       & $\radem$            	     \\
	soft falsifiability         & $\falsewt$          	     &
	covering number             & $\cover$            	     \\
	hard falsifiability         & $\falsect$          	     &
	VC-dimension                & $\vc$                	     \\
	set of hypotheses           & $\hypotheses$       	     &
	Littlestone dimension       & $\ldim$              	     \\
	theory                      & $\theory$           	     &
	Turing machine              & $\turing$           	     \\
	\hline
\end{tabular*}
\end{center}
We restrict to binary classification in this paper.

\section{The Bayesian information gain and the induced distribution}
\label{s:represent}

This section presents Bayesian information gain and the induced distribution. They will be used to quantify falsifiability in sections $x.3$.

Suppose that $X$ is a finite set, and that we are given a conditional distribution $\prob_\fm(y|x)$ and a prior $\prob_X$ on $X$. The conditional distribution models a noisy channel $\fm$ connecting $X$ to $Y$. 

\begin{defn}[Bayesian information gain; induced distribution]
	\label{d:prob_ei}
	The \textbf{Bayesian information gain} when $\fm$ outputs $y$ is
	\begin{equation*}
		\label{e:p_ei}
		\information\big(\fm,y,\prob_X\big) := \kl\Big[\prob_\fm(X|y)\,\Big\|\, \prob_X(X)\Big],
	\end{equation*}
	where $\kl[\prob\,\|\,\probq] := \sum_{x\in X} \prob(x)\log\frac{\prob(x)}{\probq(x)}$ is the Kullback-Leibler divergence. The posterior $\prob_\fm(x|y)$ is computed via Bayes' rule
	\begin{equation*}
		\label{e:actual}
		\prob_\fm(x|y) = \prob_\fm(y|x)\cdot \frac{\prob_X(x)}{\prob_\fm(y)},
	\end{equation*}
	where $\prob_\fm(y)=\sum_{x\in X} \prob_X(x)\prob_\fm(y|x)$ is the $\fm$-\textbf{induced distribution} on $Y$.
\end{defn}
The Bayesian information gain quantifies how much observing $y$ reduces uncertainty about $X$. We remark that
\begin{prop}
	The mutual information communicated across $\fm$ is the expected information gain
	\begin{equation*}
		I_\fm(X,Y) = \expec_{y\sim \prob_\fm(Y)} \information\big(\fm,y, \prob_X\big),
	\end{equation*}	
	where the expectation is with respect to the $\fm$-induced distribution on $Y$.
\end{prop}

\begin{rem}[uniform priors on finite sets]\label{r:uniform}
	Unless otherwise specified, finite sets are given the uniform prior: $\prob_{\unif}(x)=\frac{1}{|X|}$. We write $\information(\fm,y)$ as a shorthand for $\information(\fm,y, \prob_{\unif})$.
\end{rem}

Given a function $f:X\rightarrow Y$, define the corresponding conditional distribution
\begin{equation*}
	\prob_f(y|x) = \begin{cases}
		1 & \text{if } y = f(x) \\
		0 & \text{else.}
	\end{cases}
\end{equation*}

\begin{lem}\label{t:finite_det}
	Given a function $f:X\rightarrow Y$, the $f$-induced distribution on $Y$ is
	\begin{equation*}
		\prob_f(y) = \begin{cases}
			\frac{|f^{-1}(y)|}{|X|} & \text{if }y\in \im(f) \\
			0 & \text{else.}
		\end{cases}		
	\end{equation*}	
	The Bayesian information gain is 
	\begin{equation*}
		\information(f,y) = \begin{cases}
			-\log \prob_f(y) & \text{if }y\in\im(f)\\
			\text{undefined} & \text{else.}
		\end{cases}
	\end{equation*}
\end{lem}

\begin{lem}\label{t:uninformative}
	The information gain is zero, $\information(f,y)=0$, if and only if $f(x)=y$ for all $x\in X$.
\end{lem}

\section{Statistical learning}
\label{s:slt}

Statistical learning is concerned with inductive inference under the assumption that observations are drawn independently from an unknown, but fixed, probability distribution.

This section introduces falsifiability in detail. The later sections on sequential prediction and universal induction rely in part on the presentation developed here. 

\setcounter{subsection}{-1}
\subsection{Setup}

Let $X$ be an arbitrary set and $Y=\{0,1\}$. Let $Z=X\times Y$. A datum $z=(x,y)$ in $Z$ consists of an input $x$ and an outcome or label $y$. A process is a map $\sigma:X\rightarrow Y$ from inputs to outcomes. The hypothesis space $\hypotheses:= Y^X=\{\sigma:X\rightarrow Y\}$ is the set of all processes. Finally, an \emph{event} $(x,\sigma)$ is an element of $X\times \hypotheses$.

A \emph{theory} is a set of hypotheses, $\theory\subset \hypotheses$. Elements of the theory are referred to as predictors. Of course, by definition a predictor is also a hypothesis.

Let $\loss:\theory\times X\times Y\rightarrow \unint$ denote the 0/1 loss:
\begin{equation*}
	\loss(f,x,y) = \indicator[f(x)\neq y] = \begin{cases}
		0 & \text{if }f(x)=y \\
		1 & \text{else.}
	\end{cases}	
\end{equation*}
Predictor $f$ \emph{explains}\footnote{Clearly, we are using `explain' in a very weak, technical sense.} datum $(x,y)$ if $\loss(f,x,y)=0$. If not, then $(x,y)$ falsifies $f$.

\subsection{The risk ($\iid$)}

We assume throughout this section that the sample $\vec{x}$ contains $n$ distinct points.

Let $X^\bullet = \bigcup_{t=1}^\infty X^t$ denote the set of finite sequences of elements of $X$. We typically refer to sequences $\vec{x}=(x_1,\ldots,x_n)$ rather than sets $\{x_1,\ldots,x_n\}$ to keep notation and terminology consistent across sections.

\begin{defn_risk}[risk, $\iid$]
	The \textbf{risk} of theory $\theory$ on sequences of events is
	\begin{equation*}
		\risk^\iid_{\theory}:
		\hypotheses\times X^\bullet\rightarrow\unint:
		(\sigma,\vec{x})
		\mapsto 
		\inf_{f\in\theory}\frac{1}{n}\sum_{t=1}^n
		\loss\big(f,x_t,\sigma(x_t)\big),		
	\end{equation*}
	where $n=\len(\vec{x})$.
	The risk on distributions on data is
	\begin{equation*}
		\risk^\iid_{\theory}:
		\Delta_Z\rightarrow\unint:
		\prob_Z
		\mapsto 
		\inf_{f\in\theory}\expec_{z\sim\prob_Z}
		\loss\big(f,z\big).
	\end{equation*}
\end{defn_risk}
The risk quantifies the fraction of events that the best predictor in $\theory$ labels incorrectly -- that is, the fraction of events that the theory cannot explain:
\begin{equation*}
	\risk_\theory:\big\{\textrm{sequence of events}\big\}
	\mapsto \big\{\textrm{fraction of sequence that $\theory$ cannot explain}\big\}.
\end{equation*}
The risk is zero if and only if there is a predictor in $\theory$ that explains the entire sequence of events perfectly.

The set of hypotheses is not finite in general. However, since datasets are always finite, it turns out that the \emph{effective} set of hypotheses is finite.

\begin{defn}[effective hypotheses]
	Given a sequence $\vec{x}=(x_1,\ldots,x_n)$ of inputs, we say that two hypotheses $\sigma_1$ and $\sigma_2$ in $\hypotheses$ are equivalent
	\begin{equation*}
		\sigma_1\sim \sigma_2\textrm{ if and only if }\sigma_1(x_t) = \sigma_2(x_t)\textrm{ for all } t\in\{1,\ldots,n\}.
	\end{equation*}
	We refer to an equivalence class $[\sigma]=\{\tau\in\hypotheses\,|\,\sigma\sim\tau\}$ of hypotheses as an \textbf{effective hypothesis} and let $\hypotheses_{ef}=\{[\sigma]\,|\,\sigma\in\hypotheses\}$ denote the set of effective hypotheses. 
\end{defn}
Since $\vec{x}$ contains $n$ elements, it follows that there is a finite number ($2^n$) of effective hypotheses. 

Two hypotheses in the same equivalence class are indistinguishable on the observed data, and thus indistinguishable to the risk. Given a sequence of $n$ inputs $\vec{x}$, the risk can be written as a function taking effective hypotheses about $\vec{x}$ to errors:
\begin{equation*}
	\risk^\iid_{\theory,\vec{x}}:\hypotheses_{ef}\rightarrow\unint:[\sigma]\mapsto \inf_{f\in\theory}\frac{1}{n}\sum_{t=1}^n\loss\big(f,x_t,\sigma(x_t)\big).
	\tag{A}
\end{equation*}
Formulated in this way, the risk quantifies how well theory $\theory$ explains the action of an hypothetical process $\sigma$ on input sequence $\vec{x}$. More precisely, the risk $\epsilon = \risk_{\theory,\vec{x}}(\sigma)$ is the fraction of the inputs that the best predictor $f$ in $\theory$ misclassifies when labels are generated by $\sigma$.

\subsection{Learnability ($\iid$)}
\label{s:slt_learn}

A theory is learnable if it admits a strategy whose predictions match the theory's best \emph{post hoc} explanation. 

A strategy specifies the predictor that Forecaster will deploy in future as a function of previous events. Formally, a \emph{strategy} is a function taking a finite dataset $\vec{z}=(z_1,\ldots, z_n)\in Z^n$ to a predictor in $\theory$. Let $\Psi_n = \{Z^n\rightarrow \theory\}$ denote the set of strategies on datasets of size $n$. 

\begin{eg}[empirical risk minimization]
	A basic strategy is empirical risk minimization ($\erm$), which outputs the predictor that minimizes the training error:
	\begin{equation*}
		\psi_{\erm}:Z^n\rightarrow \theory:(z_1,\ldots, z_n)\mapsto \argmin_{f\in\theory} \frac{1}{n}\sum_{t=1}^n\loss(	f,z_t).
	\end{equation*}
\end{eg}

Following \cite{abernethy:09}, we formulate learnability via a game played between Forecaster and Nature. Forecaster picks a strategy $\psi\in\Psi^n$. Nature observes Forecaster's strategy and responds by choosing a distribution $\prob_Z\in\Delta_Z$ on events.

The value of the game is the generalization error of Forecaster's strategy on Nature's probability distribution: the difference between the predictive errors Forecaster's strategy accumulates and the explanatory errors of the \emph{theory's best predictor, judged after observing the distribution.} Formally, the value of the game is the difference between the risk $\risk_{\{\psi(\vec{z})\}}(\prob_Z)$ of the strategy $\psi(\vec{z})$ and the risk $\risk_{\theory}(\prob_Z)$ of the entire theory $\theory$. 

Forecaster aims to minimize the value; Nature aims for the opposite. The minimax value is thus
\begin{equation*}
	\error^\iid_n(\theory) 
	:= \inf_{\psi\in\Psi^n} \sup_{\prob_Z\in\Delta_Z}
	\underbrace{\Big[\expec_{\vec{z}\sim \prob_Z}\expec_{z'\sim \prob_Z}\loss\big(\psi(\vec{z}),z'\big) - \inf_{f\in\theory}\expec_{z'\sim \prob_Z}\loss(f,z')\Big]}_{\text{expected worst-case generalization error of Forecaster's best strategy}}
\end{equation*}
More concisely,
\begin{defn_predictive}[predictive risk, learnability; $\iid$]
	The minimax value of the game, or the \textbf{predictive risk} of theory $\theory$ on datasets of size $n$ is
	\begin{equation}
		\error^\iid_n(\theory) 
		= 
		\underbrace{\inf_{\psi\in\Psi^n}}_{\textrm{Forecaster's best strategy}}
		\overbrace{\sup_{\prob_Z\in\Delta_Z}}^{\textrm{Nature's worst distribution}}
		\underbrace{\Big[\expec_{\vec{z} \sim \prob_Z}\risk^\iid_{\psi(\vec{z})}(\prob_Z) - \risk^\iid_\theory(\prob_Z)\Big]}_{\textrm{strategy's generalization error on $\prob_Z$}}.
		\tag{B}
	\end{equation}
	the generalization error of Forecaster's best strategy when exposed to Nature's worst (for Forecaster) sequence of events.

	Theory $\theory$ is \textbf{learnable} if $\lim_{n\rightarrow\infty}\error_n(\theory)=0$.
\end{defn_predictive}
The predictive risk is the cost to Forecaster of not knowing what Nature will do next. It is measured against a baseline: Forecaster's best explanation of the entire sequence. The predictive risk thus separates the costs incurred due to predicting from the costs incurred due to having a theory that does not fit reality perfectly.

If theory $\theory$ is learnable then, for large $n$, the cumulative cost to Forecaster of not knowing what Nature will do next is negligible.  

Importantly, the predictive risk says nothing about the \emph{absolute} performance of Forecaster's strategy. A theory may have low predictive risk and still predict a particular sequence of events badly since the baseline -- the cost of using a theory that does not fit reality -- is subtracted.

\subsection{Falsifiability ($\iid$)}
\label{s:slt_false}

A theory is falsifiable to the extent that there are hypotheses that it \emph{cannot} explain. We quantify falsifiability in two ways.
\begin{defn_falsify}[falsifiability, $\iid$]
	Let $\probq_{\theory,\vec{x}}$ denote the $\risk^\iid_{\theory,\vec{x}}$-induced distribution on $\unint$. The \textbf{soft falsifiability} of $\theory$ on $\vec{x}$ is the expected error
	\begin{equation}
		\falsewt^\iid_n(\theory|\vec{x}) 
		:= 2\expec_{\epsilon\sim \probq_{\theory,\vec{x}}}[\epsilon]
		\quad\textrm{and}\quad
		\falsewt^\iid_n(\theory)
		:= \inf_{\vec{x}\in X^n}\falsewt^\iid_n(\theory|\vec{x}).
		\tag{C-s}
	\end{equation}
	The \textbf{hard falsifiability} of $\theory$ on $\vec{x}$ is
	\begin{equation}
		\falsect^\iid_n(\theory|\vec{x}) 
		:= \frac{1}{n}\information\Big(\risk^\iid_{\theory,\vec{x}},0\Big)
		\quad\textrm{and}\quad
		\falsect^\iid_n(\theory)
		:= \inf_{\vec{x}\in X^n}\falsect^\iid_n(\theory|\vec{x}).
		\tag{C-h}
	\end{equation}	

	A theory is \textbf{falsifiable} if $\lim_{n\rightarrow \infty} \falsewt_n(\theory)=1$ or $\lim_{n\rightarrow \infty} \falsect_n(\theory)=1$.
\end{defn_falsify}

\begin{rem}[falsifiability depends on the risk]\label{r:opt}
	Falsifiability is a property of the risk $\risk_{\theory,\vec{x}}:\hypotheses_{ef}\rightarrow\unint$. It depends \emph{directly} on the optimization problem underlying the learning scenario. 

	In contrast, capacity measures are typically presented as properties of the theory $\theory$ in such a way that their relation to the optimization problem (specifically, finding the predictor in $\theory$ that minimizes the error) is indirect.
\end{rem}
Taking the infimum over all possible datasets implies that $\falsewt^\iid_n(\theory)$ and $\falsect^\iid_n(\theory)$ measure worst-case falsifiability: the falsifiability of $\theory$ on the least falsifiable input sequence.

Soft falsifiability is closely related to Rademacher complexity, see Section~\ref{s:slt_proofs}. Similarly, hard falsifiability is closely related to the covering number, and so to the shattering coefficient and VC-dimension.

The coefficients $2$ and $\frac{1}{n}$ in Definition~C are chosen so that
\begin{lem}
	Soft and hard falsifiability take values in the interval $\unint=[0,1]$.
	\begin{enumerate}
		\item Theory $\theory$ shatters $\{x_1,\ldots,x_n\}$ if and only if $\falsewt^\iid_n(\theory|\vec{x})=\falsect^\iid_n(\theory|\vec{x})=0$.
		\vspace{2mm}
		\item Theory $\theory$ contains a single predictor if and only if $\falsewt^\iid_n(\theory)=\falsect^\iid_n(\theory)=1$ for all $n$.
	\end{enumerate}		
\end{lem}

\begin{proof}
	Straightforward.
\end{proof}

To interpret soft falsifiability, recall that the risk, (A), is function that takes an effective hypothesis $\sigma$ about $\vec{x}$ to the fraction $\error$ of the sequence that theory $\theory$ cannot explain (i.e. falsifies)
\begin{equation*}
	\risk^\iid_{\theory,\vec{x}}: \hypotheses_{ef}  \rightarrow  \unint:
	\sigma  \mapsto  \epsilon
\end{equation*}
The pre-image $\risk_{\theory,\vec{x}}^{-1}(\epsilon)\subset\hypotheses$ is the subset of hypotheses that, when applied to input sequence $\vec{x}$, cannot be explain by theory $\theory$ on fraction $\epsilon$ of $\vec{x}$. Thus, the risk-induced probability of $\epsilon\in\unint$ is the fraction of potential hypotheses that, if true, cause $\theory$ to falsify $\epsilon$ of the data: 
\begin{equation}
	\label{e:p_error}
	\probq(\epsilon) = \frac{|\risk_{\theory,\vec{x}}^{-1}(\epsilon)|}{|\hypotheses_{ef}|}.
\end{equation}
Finally, soft falsifiability is the weighted sum:
\begin{gather*}
	\begin{matrix}
		\falsewt^\iid(\theory|\vec{x})
		& = 2\sum_{\epsilon \in \unint} &\Big( & \frac{|\risk_{\theory,\vec{x}}^{-1}(\epsilon)|}{|\hypotheses_{ef}|} & \cdot & \epsilon &\Big)\\
		& = 2\sum_{\epsilon \in \unint} & \Big\{&\text{\small{fraction of effective hypotheses falsified}}\Big\} & 
	\cdot & \Big\{\text{\small{on fraction $\epsilon$ of data}}&\Big\}.
	\end{matrix}	
\end{gather*}
To interpret hard falsifiability, apply Lemma~\ref{t:finite_det} to obtain
\begin{align*}
	\information(\risk_{\theory,\vec{x}},0)
	= -\log \probq(0)
	& = \overbrace{\log\left|\hypotheses_{ef}\right|}^{\text{total \# effective hypotheses}} 
	- \quad
	\overbrace{\log\big|\risk_{\theory,\vec{x}}^{-1}(0)\big|}^{\text{\# hypotheses $\theory$ explains perfectly}} \\
	& = \Big\{\log\text{-\# of effective hypotheses that $\theory$ falsifies}\Big\}.
\end{align*}
If the inputs in $\vec{x}$ are distinct, then the number of effective hypotheses is $2^n$, so
\begin{equation*}
	\falsect^\iid_n(\theory|\vec{x}) = \frac{\Big\{\log\text{-\# of effective hypotheses that $\theory$ falsifies}\Big\}}{\log\Big\{\textrm{\# of effective hypotheses}\Big\}}
\end{equation*}
can be interpreted as the ``logarithmic fraction'' of effective hypotheses that $\theory$ falsifies.

\subsection{Falsifiable $\implies$ Learnable ($\iid$)}
\label{s:slt_l_is_f}

The main result is that falsifiability controls predictive risk:
\begin{thm_main}[main theorem, $\iid$]	
	\begin{equation*}
		\error^\iid_n(\theory) 
		\leq 1 - \falsewt^\iid_n(\theory)
		\leq d\sqrt{1-\falsect^\iid_n(\theory)},
		\tag{D}
	\end{equation*}	
	where $d=\sqrt{8}$.
\end{thm_main}

Surprisingly, the assumption that Nature is \emph{i.i.d.} is not essential to the result -- an almost identical theorem holds for sequential prediction, see section~\ref{s:seq}.

\begin{proof}
	By Proposition~\ref{t:rademacher-prob-iid}, soft falsifiability of a theory is essentially equivalent to its Rademacher complexity
	\begin{equation*}
		\falsewt^\iid(\theory|\vec{x}) = 1 - 2\radem^\iid\big(\loss(\theory)|\vec{x}\big).
	\end{equation*}
	Similarly, by Proposition~\ref{t:infgain_iid}, hard falsifiability recovers the covering number
	\begin{equation*}
		\falsect^\iid(\theory|\vec{x}) = 1 - \frac{\log \cover^\iid(\theory|\vec{x})}{n}.
	\end{equation*}
	 The result then follows by Theorem~\ref{t:rakhlin-iid}, which recalls two standard generalization bounds taken from \cite{rakhlin:14a}.
\end{proof}

\begin{rem}[vacuous bounds]\label{r:vacuous}
	Two ways in which Theorem~D can be vacuous are 
	\begin{enumerate}
		\item If a theory is completely unfalsifiable, $\falsewt_n(\theory)=0$, then Theorem~D provides no guarantees on its predictive performance no matter how \emph{well} it explains empirical data.
		\item If a theory is maximally falsifiable, $\falsewt_n(\theory)=1$, then it has zero predictive risk, no matter how \emph{badly} it explains empirical data. 
	\end{enumerate}
\end{rem}

\begin{cor_main}[falsifiability implies learnability, $\iid$]
	A theory is learnable if it is falsifiable:
	\begin{equation*}
		\lim_{n\rightarrow\infty} \error_n(\theory)=0
		\textrm{ if }
		\lim_{n\rightarrow\infty} \falsewt_n(\theory)=1
		\,\text{ or }
		\lim_{n\rightarrow\infty} \falsect_n(\theory)=1.
	\end{equation*}
\end{cor_main}

A much stronger version Theorem~D can also be shown.

\begin{thm_main_r}[data-dependent bounds, $\iid$]
	Let 
	\begin{equation*}
		\error^\iid_n(\theory|\vec{z},\prob) := \overbrace{\underbrace{\risk^\iid_{\psi_\erm(\vec{z})}(\prob)}_{\text{expected test error}} - \underbrace{\risk^\iid_{\psi_\erm(\vec{z})}(\vec{z})}_{\text{training error}}}^{\text{expected generalization error}}
	\end{equation*}
	be the expected generalization error of a predictor chosen using ERM.
	
	Suppose that $\vec{z}$ is a sequence of $n$ events drawn from probability distribution $\prob$ on $Z$. Let $\vec{x}$ refer to the same sequence, with labels stripped out. Then, for all $\delta>0$, with probability at least $1-\delta$,
	\begin{enumerate}
		\item the expected generalization error is upper bounded by
		\begin{equation*}
			\error^\iid_n(\theory|\vec{z},\prob) \leq 1 - 
			\falsewt^\iid(\theory|\vec{x}) + c\sqrt{\frac{1-\log\delta}{n}}
			\tag{D''-s}
		\end{equation*}	
		where $c = \sqrt{\frac{2}{\log e}}$.
	\item Furthermore, 
		\begin{equation*}
			\error^\iid_n(\theory|\vec{z},\prob) \leq d_1 \sqrt{1 - \falsect^\iid(\theory|\vec{x})} + d_2 \sqrt{\frac{1-\log \delta}{n}}
			\tag{D''-h}
		\end{equation*}
		where $d_1=\sqrt{\frac{6}{\log e}}$ and $d_2=\sqrt{\frac{1}{\log e}}$.
	\end{enumerate}	
\end{thm_main_r}

\begin{proof}
	Propositions~\ref{t:rademacher-prob-iid} and \ref{t:infgain_iid} connect soft and hard falsifiability to the Rademacher complexity and covering number. 

	The result then follows from Theorem~\ref{t:rademacher-bound}, which collects two theorems from \cite{boucheron:00} and \cite{bousquet:04}.
\end{proof}

Theorem~D'' is a true inductive bound, which requires the \emph{i.i.d.} assumption. It implies that the difference between the observed training error and expected test error depends on how many hypotheses \emph{about the training sequence $\vec{x}$} are falsified by theory $\theory$. 

In short, if strategy $\psi_\erm$ performs well on the training data, and theory $\theory$ falsifies many hypotheses about the training data, then the predictor chosen by $\psi_\erm$ will perform well in future, with high probability.

\subsection{Proofs ($\iid$)}
\label{s:slt_proofs}

Our first two results relate soft falsifiability to Rademacher complexity \cite{koltchinskii:01}.

\begin{defn}[Rademacher complexity]
	Define a \emph{Rademacher variable} $\zeta$ to be a random variable taking values in $\Omega=\{\pm1\}$ with equal probability. 

	Let $\vec{\zeta}=(\zeta_1,\ldots, \zeta_n)$ be Rademacher variables. The \textbf{Rademacher complexity} of theory $\theory$ on unlabeled inputs $\vec{x}=(x_1,\ldots,x_n)$ is
	\begin{equation*}
		\radem^\iid(\theory|\vec{x}) := \expec_{\vec{\zeta}}\left[
		\sup_{f\in\theory}\frac{1}{n}\sum_{t=1}^n \zeta_t\cdot f(x_t)
		\right].
	\end{equation*}
	The \textbf{Rademacher complexity of a theory with respect to a loss function} is
	\begin{equation*}
		\radem^\iid\big(\loss(\theory)|\vec{z}\big) := \expec_{\vec{\zeta}} \left[\sup_{f\in \theory}\frac{1}{n} \sum_{	t=1}^n \zeta_t\cdot \loss\big(f,(x_t,y_t)\big)\right].
	\end{equation*}
\end{defn}

\begin{lem}\label{t:radem_exp}	
	\begin{equation*}
		\expec_{\vec{\zeta}} \risk_\theory\big(\vec{x},\zeta\cdot\vec{y}\big) 
		= \frac{1}{2} - \radem^\iid\big(\loss(\theory)\,\big|\,\vec{z}\big)
		= \frac{1}{2} - \frac{1}{2}\radem^\iid\big(\theory\,\big|\,\vec{z}\big).
	\end{equation*}
\end{lem}

\begin{proof}
	For the first equality, observe that
	\begin{equation*}
		\zeta\cdot (1-2 \loss(f,z)) 
		= \begin{cases}
			+1 & \text{if } f(x)=\zeta\cdot y
			\\		
			-1 & \text{else,}
		\end{cases}
	\end{equation*}
	which implies
	\begin{equation*}
		\frac{1}{2}-\zeta\cdot\left(\frac{1}{2}-\loss(f,z)\right)
		= \loss(f,(x,\zeta\cdot y)).
	\end{equation*}
	It follows from $\inf_{f\in\theory}[-\psi(f)] = -\sup_{f\in\theory}\psi(f)$ that
	\begin{align*}
		\expec_{\vec{\zeta}} \risk_\theory\big(\vec{x},\vec{\zeta}\cdot\vec{y}\big)
		& = \expec_{\vec{\zeta}} \inf_{f\in\theory}\sum_{t=1}^n\loss(f,(\vec{x},\vec{\zeta}\cdot \vec{y})\\
		& = \expec_{\vec{\zeta}} \inf_{f\in\theory}\sum_{t=1}^n\left[\frac{1}{2}-\zeta_t\left(\frac{1}{2}-\loss(f,z_t)\right)\right]\\
		& = \frac{1}{2} - \expec_{\vec{\zeta}}\sup_{f\in\theory}\sum_{t=1}^n \zeta_t\cdot \loss(f,z_t)\\
		& = \frac{1}{2} - \radem^\iid(\loss(\theory)\,|\,\vec{z}).
	\end{align*}
	The second equality follows similarly.
\end{proof}

A corollary of Lemma~\ref{t:radem_exp} is that Rademacher complexity is independent of the labels $\vec{y}$. We therefore drop the labels from the notation and write $\radem^\iid(\theory|\vec{x})$ and $\radem^\iid(\loss(\theory)|\vec{x})$ below.

\begin{prop}[Rademacher complexity from soft falsifiability, $\iid$]\label{t:rademacher-prob-iid}	
	\begin{equation*}
		\frac{1}{2}\falsewt^\iid(\theory\,|\,\vec{x})
		 = \frac{1}{2} - \radem^\iid(\loss(\theory)\,|\,\vec{x}) 
		 = \frac{1}{2} - \frac{1}{2}\radem^\iid(\theory\,|\,\vec{x}).
	\end{equation*}
\end{prop}

\begin{proof}
	Recall that $\falsewt^\iid(\theory|\vec{x}) := 2\expec_{\epsilon\sim \probq}\big[\,\epsilon\,\big]$ where $\probq$ is the $\risk^\iid_{\theory,\vec{x}}$-induced distribution on $\unint$. The induced distribution is 
	\begin{equation*}
		\probq(\epsilon) 
		= \begin{cases}
			\frac{|\risk_{\theory,\vec{x}}^{-1}(\epsilon)|}{|\hypotheses_{ef}|} & \text{if }\epsilon \in \risk_{\theory,\vec{x}}(Y^X) \\
			0 & \text{else}.
		\end{cases}
	\end{equation*}
	By Lemma~\ref{t:radem_exp} it suffices to show that $\expec_{\vec{\zeta}} \risk_\theory\big(\vec{x},\vec{\zeta}\cdot\vec{y}\big)
		= \expec_{\epsilon\sim \probq}\big[\,\epsilon\,\big]$.
	Observe that
	\begin{equation*}
		\expec_{\vec{\zeta}} \risk_\theory\big(\vec{x},\vec{\zeta}\cdot\vec{y}\big)
		= \sum_{[\sigma]\in \hypotheses_{ef}} \frac{\risk_{\theory}(\vec{x},\sigma\circ \vec{x})}{|\hypotheses_{ef}|}
		= \sum_{\epsilon \in \im(\risk_{\theory,\vec{x}})} \epsilon \cdot \frac{|\risk_{\theory,\vec{x}}^{-1}(\epsilon)|}{|\hypotheses_{ef}|}
		= \expec_{\epsilon\sim \probq}\big[\,\epsilon\,\big].
	\end{equation*}
	as required.
\end{proof}

Next, we relate hard falsifiability to the covering number. 

\begin{defn}[covering number, $\iid$]\label{d:vc-entropy}
	Given unlabeled data $\vec{x}=(x_1,\ldots, x_n)\in X^n$ and a theory $\theory\subset Y^X$, let $q$ denote the map
	\begin{equation*}
		q_{\vec{x}}:\theory\rightarrow \bR^n:f\mapsto \big(f(x_1)\ldots f(x_n)\big)
	\end{equation*}
	taking predictors to labels. The \textbf{covering number} of $\theory$ on $\vec{x}$ is
	\begin{equation*}
		\cover^\iid(\theory|\vec{x}) := |q_{\vec{x}}(\theory)|,
	\end{equation*}
	the number of distinct labellings produced by the predictors in $\theory$ applied to $x_1,\ldots,x_n$. 
\end{defn}
The shattering coefficient and VC-dimension are discussed in Section~\ref{s:seq_to_slt}, see Definition~\ref{d:dim}.

The covering number coincides with hard falsifiability:
\begin{prop}[covering number from hard falsifiability, $\iid$]\label{t:infgain_iid}
	The hard falsifiability of theory $\theory$ on $\vec{x}$ is 
	\begin{equation*}
		\falsect^\iid(\theory| \vec{x}) = 1 - \frac{1}{n}\log \cover^\iid(\theory|\vec{x}).
	\end{equation*}
\end{prop}

\begin{proof}
	By definition,
	\begin{equation*}
		\information(\risk_{\theory,\vec{x}},0) = -\log \frac{|\risk_{\theory,\vec{x}}^{-1}(0)|}{|\hypotheses_{ef}|}.
	\end{equation*}
	Since the sample contains $n$ distinct points and $|Y|=2$, it follows that $\log|\hypotheses_{ef}|=n$. It is easy to check that $|q_x(\theory)| = |\risk_{\theory,\vec{x}}^{-1}(0)|$.
\end{proof}

\begin{thm}[Data-independent bounds in expectation]\label{t:rakhlin-iid}
	Let
	\begin{equation*}
		\radem^{\iid}_n\big(\loss(\theory)\big) 
		:= \sup_{\prob\in\Delta_Z} \expec_{\vec{z}\sim\prob}\radem^{\iid}\big(\loss(\theory)\,\big|\,\vec{z}\big),
	\end{equation*}
	where $\len(\vec{z})=n$. Then
	\begin{equation*}
		\error^\iid_n(\theory) 
		\leq 2 \radem^{\iid}_n\big(\loss(\theory)\big)
		\leq 2\sqrt{\frac{2 \cover^{\iid}_n(\theory)}{n}}.
	\end{equation*}	
\end{thm}

\begin{proof}
	\cite{rakhlin:14}.
\end{proof}

\begin{thm}[Data-dependent bounds with high probability]\label{t:rademacher-bound}	
	For all $\delta>0$, the following bounds hold with probability at least $1-\delta$,
	\begin{enumerate}
		\item The predictive risk is upper bounded by
	\begin{equation*}
		\error^\iid_n(\theory|\vec{z}) 
		\leq 2\radem^\iid\big(\loss(\theory)\big|\vec{x}\big) + c\sqrt{\frac{1-\log\delta}{n}},
	\end{equation*}
	where $c = \sqrt{\frac{2}{\log e}}$.
	\item Furthermore, 
	\begin{equation*}
		\error^\iid_n(\theory|\vec{z}) 
		\leq d_1\sqrt{\frac{\cover^\iid(\theory|\vec{x})}{n}}
		+ d_2\sqrt{\frac{1-\log\delta}{n}},
	\end{equation*}
	where $d_1=\sqrt{\frac{6}{\log e}}$ and $d_2=\sqrt{\frac{1}{\log e}}$. 
	\end{enumerate}	
\end{thm}

\begin{proof}
	\cite{bousquet:04} and \cite{boucheron:00}.
\end{proof}

\section{Sequential prediction}
\label{s:seq}

Sequential prediction is concerned with predicting a finite sequence of binary observations -- without any assumptions on how the observations are generated. The \emph{i.i.d.} assumption of statistical learning is replaced by an adversary that observes Forecaster's previous moves and responds maliciously. 

We build on the presentation in section \ref{s:slt}. The key technical difference between statistical learning and sequential prediction is the introduction of \emph{trees}, which requires us to distinguish between two notions of risk: soft and hard.

Remarkably, the main theorem has an almost identical form in both sequential prediction and statistical learning. However, the stronger data-dependent form, Theorem~D'', no longer holds, see discussion in section~\ref{s:discussion}. 

\setcounter{subsection}{-1}
\subsection{Setup}
We introduce some useful notation from \cite{rakhlin:14a}.

\begin{defn}[trees; paths]
	Let $\Omega=\{-1,+1\}$. A \textbf{$Z$-valued tree} of depth $n$ is an $n$-tuple $\vec{\bz}=(\bz_1,\ldots, \bz_n)$ of functions $\bz_t:\Omega^{t-1}\rightarrow Z$. Trees are denoted with boldface. A \textbf{path} is an element $\vec{\omega}=(\omega_1,\ldots, \omega_n)\in\Omega^n$. Combining a path $\vec{\omega}$ with a tree $\vec{\bz}$, obtains a sequence $\vec{\bz}(\vec{\omega})=(\bz_1, \bz_2(\omega_1),\ldots, \bz_n(\omega_{1:n-1}))$ of elements in $Z$. 
\end{defn}

It will be convenient to use the shorthand $\bX^t:= X^{\Omega^t} = \{\bx_t:\Omega^t\rightarrow X\}$. Let $\bX^\bullet= \bigcup_{t=1}^\infty \bX^t$ denote the set of all $X$-valued trees.

\subsection{The risk ($\seq$)}
\label{s:seq_risk}

We assume throughout this section that $\vec{\bx}$ contains a path with $n$ distinct points.

\begin{defn_risk}[risk, $\seq$]\label{d:risk_seq}
	Let $\hypotheses = Y^X=\{\sigma:X \rightarrow Y \}$ denote the set of hypotheses on $X$.
	The \textbf{risk} for sequential prediction is
	\begin{equation*}
		\label{e:risk_seq}
		\risk^\seq_{\theory}:
		\hypotheses\times(\Omega \times \bX)^\bullet \rightarrow\unint:
		(\sigma,\vec{\omega},\vec{\bx})
		\mapsto
		\inf_{f\in\theory} \frac{1}{n}\sum_{t=1}^{n}
		\loss\Big(f,\bx_t(\omega_{1:t-1}),\sigma\big(\bx_t(\omega_{1:t-1})\big)\Big),
	\end{equation*}	
	where $n=\len(\vec{\omega})=\len(\vec{\bx})$.	
\end{defn_risk}

The risk for sequential prediction differs from statistical learning in that the inputs are trees, not elements, and the choice of path in $\Omega^n$ is an additional degree of freedom. There are two obvious ways to deal with paths:
\begin{enumerate}
	\item \emph{Incorporate paths into the input by defining $\tilde{\bX}^n:= \Omega^n\times \bX^n$.}
	Given an $X$-valued tree $\vec{\bx} = (\bx_1,\ldots,\bx_n)$ and a path $\vec{\omega}\in\Omega^n$, we say that two hypotheses $\sigma$ and $\tau$ in $\hypotheses$ are equivalent 
	\begin{equation*}	
		\sigma \sim \tau\textrm{ iff }
		\sigma\big(\bx_t(\omega_{1:t-1})\big) = \tau\big(\bx_t(\omega_{1:t-1})\big)\quad \forall t\in\{1,\ldots,n\}.
	\end{equation*}
	Define the \textbf{soft risk},
	\begin{equation*}
		\risk^\seq_{\theory,(\vec{\omega}, \vec{\bx})}:
		\hypotheses_{ef} \rightarrow\unint:
		\sigma
		\mapsto 
		\inf_{f\in\theory}
		\left[ \frac{1}{n}\sum_{t=1}^{n} 
		\loss\Big(f,\bx_t(\omega_{1:t-1}),\sigma\big(\bx_t(\omega_{1:t-1})\big)\Big)\right].
		\tag{A-s}
	\end{equation*}
	\item \emph{Incorporate paths into the hypotheses by defining $\tilde{\hypotheses}:=\hypotheses\times \Omega^n$.}
	Similarly, two hypotheses $(\sigma,\vec{\omega})$ and $(\tau,\vec{\rho})$ in $\tilde{\hypotheses}=\hypotheses\times \Omega^n$ are equivalent 
	\begin{equation*}	
		(\sigma,\vec{\omega})\sim (\tau,\vec{\rho})\textrm{ iff }
		\sigma\big(\bx_t(\omega_{1:t-1})\big) = \tau\big(\bx_t(\rho_{1:t-1})\big) \quad\forall t\in\{1,\ldots,n\}.
	\end{equation*}
	Let $\tilde{\theory}=\theory\times\Omega^n$ and define the \textbf{hard risk},
	\begin{equation*}
		\risk^\seq_{\tilde{\theory},\vec{\bx}}:
		\tilde{\hypotheses}_{ef} \rightarrow\unint:
		(\sigma,\vec{\rho})
		\mapsto 
		\inf_{(f,\vec{\omega})\in\tilde{\theory}} 
		\left[\frac{1}{n}\sum_{t=1}^{n}
		\loss\Big(f,\bx_t(\omega_{1:t-1}),\sigma\big(\bx_t(\rho_{1:t-1})\big)\Big)\right].
		\tag{A-h}
	\end{equation*}
\end{enumerate}

\subsection{Learnability ($\seq$)}
\label{s:seq_learn}

Consider the following game played between Forecaster and Nature over $n$ rounds \cite{abernethy:09,rakhlin:14a}. 

In the first round, Forecaster chooses a probability distribution $\prob_1\in\Delta_\theory$ on the set of predictors. Nature observes Forecaster's choice, and picks $z_1\in Z$. A predictor $f_1$ is then sampled at random from $\prob_1$, applied to $z_1$ and the loss $\loss(f_1,z_1)$ is computed. The game continues for $n$ rounds, where both Forecaster and Nature observe the moves played in previous rounds. 

The value of the game is Forecaster's \textit{regret}: the difference between Forecaster's cumulative loss and the loss Forecaster would have accumulated, had it played the best move in hindsight. Forecaster's goal is to minimize its regret; Nature's aims for the opposite:
\begin{equation*}
		\label{e:value_expand}
		\error^\seq_n(\theory) = 
		\inf_{\prob_1\in\Delta_\theory}\sup_{z_1\in Z}\expec_{f_1\sim \prob_1} \cdots
		\inf_{\prob_n\in\Delta_\theory}\sup_{z_n\in Z}\expec_{f_n\sim \prob_n}
		\frac{1}{n}\underbrace{\left[\sum_{t=1}^n \loss(f_t,z_t) - \inf_{f\in\theory}\sum_{t=1}^n\loss(f,z_t)\right]}_{\text{Forecaster's regret}}
	\end{equation*}

Forecaster's move at time $t$ depends on the prior moves by Forecaster and Nature. Forecaster's strategy at time $t$ can be expressed as a function $\psi_t:Z^{t-1}\rightarrow \theory$. Let $\Psi_t=\{\psi_t:Z^{t-1}\rightarrow \theory\}$ denote the strategies available to Forecaster at time $t$, and let $\Psi = \prod_{t=1}^n \Psi_t$ denote the strategies available to Forecaster over an $n$-round game.

Similarly, Nature's strategy at time $t$ is an element of $\Xi_t=\theory^{t-1}\times\Delta_\theory\rightarrow Z$. Let $\Xi = \prod_{t=1}^n\Xi_t$ denote the $n$-round strategies available to Nature. We can write the minimax value more compactly as
\begin{equation*}
	\label{e:gmm_value}
	\error_n^\seq(\theory) = \inf_{\prob \in \Delta_\Psi} \sup_{\vec{\xi}\in \Xi} \expec_{\vec{\psi}\sim \prob}
	\frac{1}{n}\left[\sum_{t=1}^n \loss\big(\psi_t(\xi_{1:t-1}),\xi_t(\psi_{1:t-1},\prob_t)\big) - \inf_{f\in \theory}\sum_{t=1}^n \loss\big(f,\xi_t(\psi_{1:t-1}),\prob_t\big) \right],
\end{equation*}
where the $\sup$ and $\inf$ are understood to unravel recursively as above. 

Finally, substituting in the risk obtains

\begin{defn_predictive}[predictive risk, $\seq$]
	The minimax value of an $n$-round game, or \textbf{predictive risk} of theory $\theory$, is
	\begin{equation}
		\error_n^\seq(\theory) = \inf_{\prob \in \Delta_\Psi} \sup_{\vec{\xi}\in \Xi} \expec_{\vec{\psi}\sim \prob}
		\Big[ \risk^\seq_{\vec{\psi}}\big(\vec{\xi}\big) - \risk^\seq_\theory(\vec{\xi}) \Big].
		\tag{B}
	\end{equation}
	Theory $\theory$ is \textbf{learnable} if $\lim_{n\rightarrow \infty} \error_n^\seq(\theory) = 0$.
\end{defn_predictive}

The first term, $\risk_{\vec{\psi}}(\vec{\xi})$ is the cumulative loss incurred by the best $\theory$-based strategy played out on Nature's sequence of moves $\vec{\xi}$. The comparator term, $\risk_\theory(\vec{\xi})$ is the performance of the best predictor in $\theory$, taken in hindsight. 

\subsection{Falsifiability ($\seq$)}
\label{s:seq_false}

We use the soft and hard risk to define soft and hard falsifiability:

\begin{defn_falsify}[falsifiability, $\seq$]
	Let $\probq_{\theory,(\vec{\omega},\vec{\bx})}$ be the $\risk^\seq_{\theory, (\vec{\omega},\vec{\bx})}$-induced distribution on $\unint$.
	The \textbf{soft falsifiability} of theory $\theory$ on $\vec{\bx}$ is the expected error of the soft risk
	\begin{equation}
		\falsewt^\seq_n(\theory|\vec{\bx})
		:= 2\expec_{\vec{\omega}\sim \prob_{\textrm{unif}}(\Omega^n)}
		\expec_{\epsilon\sim\probq_{\theory,(\vec{\omega},\vec{\bx})}}[\epsilon]
		\quad\textrm{and}\quad
		\falsewt^\seq_n(\theory)
		:= \inf_{\vec{\bx}\in \bX}\falsewt^\seq_n(\theory|\vec{\bx}).
		\tag{C-s}
	\end{equation}
	The \textbf{hard falsifiability} of theory $\theory$ on $\vec{\bx}$ is the information gain from the hard risk
	\begin{equation}
		\falsect^\seq_n(\theory|\vec{\bx})
		:= \frac{1}{n}\information(\risk^\seq_{\theory\times\Omega^n,\vec{\bx}},0)
		\quad\textrm{and}\quad
		\falsect^\seq_n(\theory) := \inf_{\vec{\bx}\in \bX}\falsect^\seq_n(\theory|\vec{\bx}).
		\tag{C-h}
	\end{equation}
	A theory is \textbf{falsifiable} if $\lim_{n\rightarrow \infty} \falsewt_n(\theory)=1$ or $\lim_{n\rightarrow \infty} \falsect_n(\theory)=1$.
\end{defn_falsify}

Hard falsifiability is closely related to the sequential covering number introduced in \cite{rakhlin:14a}. However, the definition is more intuitive and, importantly, it also leads to combinatorial bounds such as the Littlestone dimension, see Section \ref{s:seq_proofs} for details. 

\subsection{Falsifiable $\implies$ Learnable ($\seq$)}
\label{s:seq_l_is_f}

Finally, we obtain the main theorem for sequential prediction, which is an exact analog of the corresponding theorem for statistical learning:

\begin{thm_main}[main theorem, $\seq$]
	\begin{equation}
		\tag{D}
		\error_n^\seq(\theory) 
		\leq 1 - \falsewt^\seq_n(\theory) 
		\leq c\sqrt{1 - \falsect^\seq_n(\theory)}
	\end{equation}
	where $c=\sqrt{8}$.
\end{thm_main}

An important point is that hard falsifiability provides a \emph{non-vacuous} upper-bound for the zero-covering number, see Section~\ref{s:seq_to_slt}.

\begin{proof}
	By Proposition~\ref{t:rademacher-prob-seq}, soft falsifiability is equivalent to the sequential Rademacher complexity
	\begin{equation*}
		\falsewt^\seq(\theory|\vec{\bx}) = 1 - 2\radem^\seq\big(\loss(\theory)|\vec{\bx})\big).
	\end{equation*}
	The first inequality then follows from Theorem~\ref{t:rakhlin_radem}, taken from \cite{rakhlin:14a}. 

	By Lemma~\ref{t:infgain_seq} and Proposition~\ref{t:falsect_seq}, hard falsifiability can be used to upper bound the sequential zero-covering number:
	\begin{equation*}
		\frac{\cover^\seq(\theory|\vec{\bx})}{n} \leq 1-\falsect^\seq(\theory|\vec{\bx}).
	\end{equation*}
	The second inequality then follows from Theorem~\ref{t:rakhlin_vc}, also taken from \cite{rakhlin:14a}. 
\end{proof}

\begin{cor_main}[falsifiability implies learnability, $\seq$]
	A theory is learnable if it is falsifiable:
	\begin{equation*}
		\lim_{n\rightarrow\infty} \error_n(\theory)=0
		\textrm{ if }
		\lim_{n\rightarrow\infty} \falsewt_n(\theory)=1
		\,\text{ or }
		\lim_{n\rightarrow\infty} \falsect_n(\theory)=1.
	\end{equation*}
\end{cor_main}

\subsection{Proofs ($\seq$)}
\label{s:seq_proofs}

This section proves the falsification bounds in Theorem~D for sequential prediction. 

\begin{defn}[Sequential Rademacher complexity]\label{d:sradem}
	\begin{equation*}		
		\radem^\seq(\theory|\vec{\bx}) 
		:= \expec_{\vec{\zeta}}\left[\sup_{f\in\theory}\frac{1}{n}\sum_{t=1}^n \zeta_t f(\bx_t(\zeta_{1:t-1})) \right]
	\end{equation*}
\end{defn}

\begin{prop}[Rademacher complexity from induced distribution, $\seq$]\label{t:rademacher-prob-seq}
	Let $\probq_{\vec{\omega}} := \prob_{\risk^\seq_{\theory,(\vec{\omega},\vec{\bx})}}$ be the distribution on errors in $\unint$ induced by the soft risk $\risk^\seq_{\theory,(\vec{\omega},\vec{\bx})}:\hypotheses\rightarrow \unint$. Then,
	\begin{equation*}
		\radem^\seq(\loss(\theory),\vec{\bx}) = \frac{1}{2} - \expec_{\vec{\omega}\sim \prob_{\unif}(\Omega^n)}\expec_{\epsilon\sim \probq_{\vec{\omega}}}\big[\,\epsilon\,\big].
	\end{equation*}
\end{prop}

\begin{proof}
	As for Proposition~\ref{t:rademacher-prob-iid}.
\end{proof}

\begin{thm}\label{t:rakhlin_radem}
	The predictive risk of sequential prediction is bounded by
	\begin{equation*}
		\error^\seq_n(\theory) \leq 2\sup_{\vec{\bx}\in\bX}\radem^\seq\big(\loss(\theory),\vec{\bx}\big),
	\end{equation*}
	where the $\sup$ is over trees of length $n$.
\end{thm}

\begin{proof}
	\cite{rakhlin:14a}.
\end{proof}

Next, we upper bound the covering number of a tree-process. The following definition is given in \cite{rakhlin:14a}

\begin{defn}[covering number, $\seq$]\label{d:zero-cover}
	A \textbf{zero-cover} of $\theory$ on an $X$-valued tree $\vec{\bx}$ is a set $V$ of $Y$-valued trees such that
	\begin{equation*}
		\forall f\in\theory,\, \forall(\omega_1,\ldots, \omega_n)\in \Omega^n,\, \exists\bv\in V \text{ s.t. }
		f(\bx_t(\omega_{1:t-1})) = \bv_t(\omega_{1:t-1}) \,\,\forall t\in \{1,\ldots, n\}.
	\end{equation*}
	The \textbf{covering number} of $\theory$ on $\bx$ is
	\begin{equation*}
		\cover^\seq(\theory,\vec{\bx}) = \min\{|V| : V \text{ is a zero-cover}\}.
	\end{equation*}
\end{defn}

The sequential covering number is awkward for our purposes since, unlike the statistical covering number in Definition~\ref{d:vc-entropy}, it is not defined as the cardinality of the image of a function. We therefore need the following 

\begin{lem}[upper bound for sequential covering number]\label{t:infgain_seq}
	Let 
	\begin{equation*}
		q_{\vec{\bx}}:\ttheory\rightarrow \bR^n:(f,\vec{\omega})\mapsto \Big(f(\bx_1),f(\bx_2(\omega_1),\ldots, f(\bx_n(\omega_{1:n-1})\Big)
	\end{equation*}
	The covering number is upper bounded by 
	\begin{equation*}
		\cover^\seq(\theory,\vec{\bx}) \leq |q_{\vec{\bx}}(\ttheory)|.
	\end{equation*}
\end{lem}

\begin{proof}
	We prove the lemma by constructing a zero-cover $V_q$ of $\theory$ on $\vec{\bx}$ with $|q_{\vec{\bx}}(\ttheory)|$ elements.

	Suppose the image $q_{\vec{\bx}}(\theory\times\Omega^n)$ has ${\mathcal N}$ elements, ${\mathbf q}^1,\ldots, {\mathbf q}^{\mathcal N}$. Define
	\begin{equation*}
		\bv^j(\omega_{1:t-1}) := {\mathbf q}^j_t.
	\end{equation*}
	That is, $\bv^j(\omega_{1:t-1})$ is the $t^{\textrm{th}}$ element of ${\mathbf q}^j$ for all paths in $\Omega^n$. Then, by construction $V_q=\{\bv^1,\ldots, \bv^{\mathcal N}\}$ is a zero-cover of $\vec{\bx}$ containing ${\mathcal N}$ elements, and we are done.
\end{proof}

\begin{prop}\label{t:falsect_seq}
	\begin{equation*}
		\information(\risk^\seq_{\ttheory,\vec{\bx}},0) = n - \log |q_{\vec{\bx}}(\ttheory)|.
	\end{equation*}	
\end{prop}

\begin{proof}
	As for Proposition~\ref{t:infgain_iid}.
\end{proof}

\begin{thm}\label{t:rakhlin_vc}
	Let $\vec{\bx}$ be an $X$-valued tree of length $n$. Then,
	\begin{equation*}
		\radem^\seq(\theory,\vec{\bx}) \leq \sqrt{\frac{2\log \cover^\seq(\theory,\vec{\bx})}{n}}
	\end{equation*}
\end{thm}

\begin{proof}
	\cite{rakhlin:14a}.
\end{proof}

It follows from Lemma~\ref{t:infgain_seq}, Proposition~\ref{t:falsect_seq} and Theorem~\ref{t:rakhlin_vc} that hard falsifiability can be used to upper bound the predictive risk for sequential prediction.

\subsection{A sequential-to-statistical reduction}
\label{s:seq_to_slt}

Definition~\ref{d:zero-cover}, of the sequential covering number, is fairly intricate and fragile. For example, slightly changing the definition by reordering the quantifiers gives a quantity that grows much too fast and yields vacuous generalization bounds \cite{rakhlin:14}. 

A natural concern is therefore that the upper bound in Lemma~\ref{t:infgain_seq} is too loose. In the remainder of this section, we show that $|q_{\vec{\bx}}(\ttheory)|$, and so hard falsifiability, is a \emph{useful}, non-vacuous upper bound.

\begin{defn}[shattering, VC and Littlestone dimensions]\label{d:dim}
	We have the following analogous definitions:
	\begin{enumerate}
		\item \emph{Statistical.}\\
		Theory $\theory$ \textbf{shatters} input sequence $\vec{x}$ of length $n$ if 
		\begin{equation*}
			\forall \vec{\omega}\in\Omega^n\quad
			\exists f\in\theory\quad \textrm{ s.t. }\quad
			f(x_t) = \frac{\omega_t+1}{2}\quad
			\forall t\in\{1,\ldots, n\}.
		\end{equation*}
		Alternatively, $\theory$ shatters $\vec{x}$ if $\cover^\iid(\theory|\vec{x})=2^n$.
		The \textbf{VC-dimension} is
		\begin{equation*}
			\vc(\theory) := \sup\big\{n\,\big|\,\exists\text{ input sequence }\vec{x} \text{ of length }n\text{ s.t. }\theory \text{ shatters }\vec{x}\big\}
		\end{equation*}
		\item \emph{Sequential.}\\
		Theory $\theory$ \textbf{$\seq$-shatters} tree $\vec{\bx}$ of length $n$ if
		\begin{equation*}
			\forall \vec{\omega}\in\Omega^n\quad
			\exists f\in\theory\quad \textrm{ s.t. }\quad
			f\big(\bx_t(\omega_{1:t-1})\big) = \frac{\omega_t+1}{2}\quad
			\forall t\in\{1,\ldots, n\}.
		\end{equation*}
		The \textbf{Littlestone dimension} is
		\begin{equation*}
			\ldim(\theory) = \sup_{\vec{\bx}} \big\{n \,\big|\, \exists \text{$X$-valued tree $\vec{\bx}$ of length $n$ s.t. $\theory$ $\seq$-shatters }\vec{\bx}\big\}.
		\end{equation*}
	\end{enumerate}
\end{defn}

Let $Y^{\bX^\bullet} := \{\sigma:\bX^\bullet\rightarrow Y\}$ denote the set of hypotheses on the set $\bX^\bullet$ of $X$-valued trees. Given theory $\theory\subset Y^X$, define the \emph{new theory}
\begin{equation*}
	\ttheory:=\theory\times\Omega^\bullet\subset Y^{\bX^\bullet}:(f,\vec{\omega})(\bx_t) = f\big(\bx_t(\omega_{1:t-1})\big).	
\end{equation*}
The lifted theory $\tilde{\theory}$ acts on trees, which from our point of view are just another set. The statistical covering number for $\tilde{\theory}$ is given, following Definition~\ref{d:vc-entropy}m using the function,
\begin{equation*}
	q_{\vec{\bx}}:\tilde{\theory}\rightarrow \bR^n:(f,\vec{\omega})\mapsto \Big((f,\vec{\omega})(\bx_1),\ldots, (f,\vec{\omega})(\bx_n)\Big)
\end{equation*}
with $\cover^\iid(\tilde{\theory}|\vec{\bx}) = |q_{\vec{\bx}}(\tilde{\theory})|$. The VC-dimension of $\tilde{\theory}$ is then computed straightforwardly.

\begin{prop}[VC-dimension lower bounds Littlestone dimension]
	The Littlestone dimension of $\theory$ is lower-bounded by the VC-dimension of the lifted theory $\tilde{\theory}=\theory\times\Omega^\bullet$:
	\begin{equation*}
		\vc(\ttheory) \leq \ldim(\theory).
	\end{equation*}
\end{prop}

The proposition shows that the Littlestone dimension can be recovered from hard falsifiability. Thus, hard falsifiability can play the same role as the sequential covering number in reducing learning problems into combinatorial problems.

\begin{proof}
	Suppose there is a tree $\vec{\bx}$ of length $n$ shattered by $\tilde{\theory}$. We construct a new tree $\vec{\mathbf z}$ of length $n$ that is $\seq$-shattered by $\theory$.

	Thus, we assume that
	\begin{equation}
		\label{e:ldim}
		\forall(\omega_1,\ldots,\omega_n)\in \Omega^n,\, \exists (f,\vec{b})\in\ttheory
		\quad\text{ s.t. }\quad
		f\big(\bx_t(b_{1:t-1})\big) = \frac{\omega_t+1}{2}\quad
		\forall t\in\{1,\ldots, n\}.
	\end{equation}
	Let $\alpha$ denote the function specified by $\alpha(\omega_{1:t-1})=b_{1:t-1}$, as in \eqref{e:ldim}. Construct the new tree $\vec{\mathbf z}$ by $\vec{\mathbf z}=\vec{\bx}\circ \alpha$. It follows, by the construction of $\alpha$ and by \eqref{e:ldim}, that $\forall(\omega_1,\ldots,\omega_n)\in \Omega^n,\, \exists f\in\theory$ such that
	\begin{equation*}
		f\big(\vec{\mathbf z}(\omega_{1:t-1})\big) 
		= f\big(\bx_t\circ\alpha(\omega_{1:t-1})\big) 
		= f\big(\bx_t(b_{1:t-1})\big) 
		= \frac{\omega_t+1}{2}\quad
		\forall t\in\{1,\ldots, n\}
	\end{equation*}
	as required.
\end{proof}

The following instructive example, taken from \cite{rakhlin:14}, was designed to exhibit the intricacy of the sequential covering number's definition. We conclude by computing the statistical covering number of $\ttheory$ on the example, and showing that it yields the correct result. 

\begin{eg}
	Consider the function class
	\begin{equation*}
		\theory = \{f_a \,|\, a\in \unint,\, f_a(x)=0\,\, \forall x\neq a,\, f_a(a)=1\}\subset Y^\unint.
	\end{equation*}
	Assuming that the tree $\vec{\bx}$ takes on $2^{n-1}$ distinct values (the ``worst case''), then for any ordered pair $(f,\vec{\omega})$ we have that
	\begin{equation*}
		q_{\vec{\bx}}(f_a,\vec{\omega}) = \Big(f_a(\bx_1),f_a(\bx_2(\omega_{1}),\ldots, f_a(\bx_n(\omega_{1:n-1})\Big)
	\end{equation*}
	is either equal to all zeros, or all zeros with a single coordinate that equals one. The image of $q_{\vec{\bx}}$ therefore contains at most $n+1$ points and in fact $|q_{\vec{\bx}}(\ttheory)|=n+1$.
\end{eg}

\section{Universal induction}
\label{s:uni}

The third setting is universal induction, which is concerned with predicting computable sequences of binary observations. The setting differs significantly from statistical learning and sequential prediction. For example, universal induction cannot be modeled adversarially since both Nature and Forecaster have too many degrees of freedom. 

There are at least two interpretations of universal induction:
\begin{enumerate}[U1.]
	\item \emph{Universal.}
	Forecaster has a single, universal theory.
	\item \emph{Adaptive.}
	Forecaster constructs a series of theories in response to successive observations. 
\end{enumerate}
The first interpretation is standard. The second, which we advocate here, is new. Both are legitimate.

Under the first interpretation, it does not make sense to evaluate the falsifiability of theories -- since there is only one theory and it is universal. The only choice that matters is Nature's choice of sequence $\vec{y}$. It then turns out that the number of hypotheses Nature falsifies (eliminates) whilst choosing $\vec{y}$ controls Forecaster's predictive risk, see section~\ref{s:uni_std}. 

Under the second interpretation, developed in detail below, Forecaster's predictive risk is controlled by the number of hypotheses that Forecaster falsifies whilst adapting its theories.

\setcounter{subsection}{-1}
\subsection{Setup}

Let $\X$ denote the set of valid programs, where valid programs $\X\subset \bigcup_{t=1}^\infty \{0,1\}^t$ form a prefix-free set. A \emph{prefix-free} universal Turing machine $\turing$ takes valid programs to outputs. Let $Y^\infty = \{0,1,00,01,10,11,000,\ldots\}$ denote the set of all binary sequences, of finite or infinite length. 
A Turing machine is a function
\begin{equation*}
	\label{e:turing}
	\turing:\X\rightarrow Y^\infty.
\end{equation*}
Let $\Y=\turing(\X)\subset Y^\infty$ denote the set of computable sequences.

Prefix free strings formalize the notion of a computer program. For example, the set of valid C++ programs is a prefix free set since C++'s syntax ensure one program cannot be the prefix of another. The set of valid programs has a complicated structure, since it includes strings of varying length. 

It is mathematically convenient to force programs to have a fixed length. First, let 
\begin{equation*}
	\X^n = \{\vec{x}\in\X |\len(\vec{x}) = t \text{ for some } t\leq n\}.
\end{equation*}
Second, \emph{pad out} short programs: given a program $\vec{x}$ of length $t<n$, construct $2^{n-t}$ programs of length $n$ by adding arbitrary suffixes to $\vec{x}$. For example, if $\len(\vec{x})=n-2$, then the four padded programs are $\{\vec{x}00, \vec{x}01, \vec{x}10, \vec{x}11\}$. The Turing machine ignores the padding. Concretely, a C++ compiler would also ignore the padding, so the padded-out programs are all functionally equivalent.

Let $\hypotheses^n$ denote the set of binary strings of length $\leq$ $n$ and let $\theory^n \subset \hypotheses^n$ denote the set of valid, padded programs of length $n$. Denote the function that strips out the padding by
\begin{equation*}
	\cS^n:\hypotheses^n\rightarrow \X\cup\{\emptyset\}:\vec{h}\mapsto 
	\begin{cases}
		\vec{x} & \text{if } \theory^n\ni\vec{h}=\vec{x}\vec{s}\text{ for $\vec{x}$ a valid program with padding }\vec{s}\\
		\emptyset & \text{else.}
	\end{cases}	
\end{equation*}
In other words, if the string contains a valid program as prefix, then $\cS^n$ strips out the padding. If the string does not contain a valid program, then $\cS^n$ outputs a null character.

The reason for introducing padded strings is that it allows the following simple description of the Solomonoff prior as a limit distribution, induced by the uniform distribution on padded strings:

\begin{def_prop}[Solomonoff prior]\label{t:limit}
	Equip $\hypotheses^n$ with the uniform distribution for all $n$. Let $\prob_n$ denote the $\cS^n$-induced distribution on $\X\cup\{\emptyset\}$. Then
	\begin{equation*}
		\prob_{\cS}(\vec{x}) := \lim_{n\rightarrow\infty} \prob_n(\vec{x}) = 2^{-\len(\vec{x})}
	\end{equation*}
	for all $\vec{x}\in \X$.

	Let $\probq_n$ denote the $(\cS^n\circ \turing)$-induced distributed on $Y^\infty$.
	The \textbf{Solomonoff prior} is 
	\begin{equation*}
		\probq_\sol(\vec{y}) 
		:= \lim_{n\rightarrow \infty}\probq_n(\vec{y}) 
		= \sum_{\{\vec{x}|\turing(\vec{x}) = \vec{y}\bullet\}} 2^{-\len(\vec{x})}.
	\end{equation*}
\end{def_prop}

\begin{proof}
	The standard definition of the Solomonoff prior, and a demonstration that our definition coincides with the standard, are provided in section~\ref{s:uni_proofs}.
\end{proof}

Proposition~\ref{t:limit} allows us to consider how Solomonoff induction acts on inputs to the Turing machine, instead of its outputs.

\subsection{The risk ($\uni$)}
\label{s:uni_risk}

For universal induction, the loss compares the sequences generated by Nature and Forecaster element-wise:
\begin{equation*}
	\loss:Y\times Y\rightarrow \bR:(y,y') \mapsto \indicator[y\neq y'],
\end{equation*}
where as above $Y=\{0,1\}$.

\begin{defn_risk}[risk, $\uni$]
	The \textbf{risk} for universal induction is
	\begin{equation*}
		\risk^{n}:\hypotheses^n\times \hypotheses^n\rightarrow \bR_{\geq0}:
		(\vec{x},\vec{f}) \mapsto \sum_{t=1}^\infty \loss\big(\turing(\vec{x})_t,\turing(\vec{f})_t\big)
	\end{equation*}
	The \textbf{risk} of theory $\theory^n := \X^n$ is 
	\begin{equation*}
		\risk_{\theory^n}^\uni:\hypotheses\rightarrow \bR_{\geq 0}:
		\vec{x}\mapsto \inf_{\vec{f}\in \theory^n}\sum_{t=1}^\infty \loss\big(\turing(\vec{f})_t,\turing(\vec{x})_t\big).
	\end{equation*}
\end{defn_risk}

As for statistical learning and sequential prediction, we reinterpret the risk as a function from hypotheses -- that is, programs with length at most $n$ -- to nonnegative reals
\begin{equation*}
	\risk_{\vec{y}}^{n}:\hypotheses^n\rightarrow \bR_{\geq0}:
	\vec{x}\mapsto\sum_{t=1}^\infty \loss\big(\turing(\vec{x})_t,y_t\big).
	\tag{A}
\end{equation*}
In the limit we obtain $\risk^\uni_{\vec{y}} := \lim_{n\rightarrow\infty} \risk^{n}_{\vec{y}}$ as a function $\risk_{\vec{y}}^{\uni}:\hypotheses\rightarrow \bR_{\geq0}$.

\subsection{Learnability ($\uni$)}
\label{s:uni_learn}

Suppose that Nature chooses a sequence $\vec{y}\in \Y$ and reveals $\vec{y}_{1:t-1}=(y_1,\ldots, y_{t-1})$ at time $t$. Let $\psi_t = \{\psi_t:Y^{t-1}\rightarrow \Delta_Y\}$ denote the set of strategies available to Forecaster in round $t$, and $\Psi=\prod_{t=1}^\infty \psi_t$ the set of all strategies available to Forecaster. 

The risk of strategy $\psi$ is
\begin{equation*}
	\risk^\uni_\psi:\hypotheses\rightarrow\bR_{\geq0}:\vec{x}\mapsto \sum_{t=1}^\infty \expec\loss\Big(\psi_t\big(\turing(\vec{x})_{1:t-1}\big),\turing(\vec{x})_t\Big),
\end{equation*}
where the expectation is over the outputs of the (probabilistic) strategy.

A particularly important strategy is \emph{Solomonoff induction} \cite{solomonoff:64}: 

\begin{def_prop}[Solomonoff induction]\label{t:sol_strategy}
	Let 
	\begin{equation*}
		\theory^n_t := (\risk^{n}_{y_{1:t-1}})^{-1}(0)
		= \big\{\text{hypotheses of length $\leq n$ that explain $y_{1:t-1}$}\big\}.
	\end{equation*}
	Theory $\theory^n_t$ is a finite set; equip it with the uniform distribution. Let $\prob_{n,t}(\vec{x})$ denote the $\cS^n$-induced distribution on $\X$ and $\probq_{n,t}(\vec{y})$ denote the $(\cS^n\circ \turing)$-induced distribution on $Y^\infty$. 
	
	\textbf{Solomonoff induction} is the strategy:
	\begin{equation*}
		(\psi_\sol)_t:Y^{t-1}\rightarrow \Delta_Y:y_{1:t-1}
		\mapsto 
		\lim_{n\rightarrow \infty} \probq_{n,t}(y_t) = \probq_\sol(y_t|y_{1:t-1}).
	\end{equation*}
\end{def_prop}

Solomonoff induction depends on the choice of Turing machine, although this dependence is typically not explicit in our notation.

\begin{proof}
	We show that $\lim_{n\rightarrow \infty} \probq_{n,t}(y_t) = \probq_\sol(y_t|y_{1:t-1})$ in section~\ref{s:uni_proofs}.
\end{proof}

Solomonoff induction can be interpreted as follows. Forecaster's theory at time step $t$ is $\theory_t:=\lim_{n\rightarrow\infty}\theory^n_t$, a limit of finite sets. All hypotheses consistent with the previous observations $y_{1:t-1}$ are weighted equally (recalling that padding entails redundancies). Forecaster predicts the next observation by drawing from $\theory_t$ uniformly at random. After observing $y_t$, and regardless of whether or not Forecaster's prediction at time $t$ was correct, Forecaster constructs new theory $\theory_{t+1}$ in the light of $y_t$.

In short, Solomonoff induction learns by constructing a nested set of progressively smaller theories and predicts by sampling from them uniformly at random.

\begin{defn_predictive}[predictive risk, $\uni$]
	The \textbf{predictive risk} of strategy $\psi$ and theory $\theory^n$ is
	\begin{equation*}
		\error^\uni(\psi-\theory^n|\vec{y}) 
		:= \risk_\psi^\uni\big(\vec{y}\big) - \risk_{\theory^n}^\uni\big(\vec{y}\big)
	\end{equation*}
	The \textbf{predictive risk} of strategy $\psi$ is
	\begin{equation}
		\error^\uni(\psi|\vec{y}) 
		:= \lim_{n\rightarrow\infty}\error^\uni_\psi(\theory^n|\vec{y}) .
		\tag{B}
	\end{equation}
\end{defn_predictive}

\subsection{Falsifiability ($\uni$)}
\label{s:uni_false}

This subsection and the next relate the error accumulated using Solomonoff induction to the falsifiability of the string chosen by Nature. 

\begin{defn_falsify}[falsifiability, $\uni$]
	\begin{equation}
		\tag{C-h}
		\falsect_\turing^\uni(\vec{y}) 
		:= \lim_{n\rightarrow\infty}\information(\risk_{\vec{y}}^{n}, 0).
	\end{equation}
\end{defn_falsify}

\begin{rem}
	The definition for universal induction differs from statistical learning and sequential prediction, in that the coefficient $\frac{1}{n}$ is not present, and so $\falsect^\uni$ does not necessarily  take values in $[0,1]$.
\end{rem}

To interpret hard falsifiability, first fix an ambient hypothesis space $\hypotheses^n$, and consider the hypotheses falsified when observing the substring $y_{1:t}$:
\begin{align*}
	\falsect^n_\turing(\vec{y}_{1:t})
	& = \log 2^n - \log|\theory^n_{t}|\\
	& = \big\{\log\text{-\# strings of length $n$}\big\}
	- \big\{\log\text{-\# strings that output $y_{1:t}$}\big\} \\
	& = \big\{\log\text{-\# strings of length $n$ falsified by }y_{1:t}\big\}.
\end{align*}
Second, consider the hypotheses eliminated when transitioning between theories:
\begin{align*}
	\log |\theory^n_{t}| - \log|\theory^n_{t-1}|
	& = 
	\big\{\log\text{-\# strings outputting $y_{1:t-1}$}\big\}
	- \big\{\log\text{-\# strings outputting $y_{1:t}$}\big\}\\
	& = 
	\big\{\log\text{-\# strings falsified when modifying }\theory_{t-1}\mapsto \theory_{t}\big\}.
\end{align*}
Finally, combining the above obtains
\begin{align*}
	\falsect^\uni_\turing(\vec{y})
	& = \sum_{t=1}^\infty \lim_{n\rightarrow\infty}
	\Big(\falsect^n_\turing(\vec{y}_{1:t})-\falsect^n_\turing(\vec{y}_{1:t-1})
	\Big)
	\quad\text{where }y_{1:0}:=\emptyset \\ 
	& = \sum_{t=1}^\infty \lim_{n\rightarrow\infty}\Big(\log |\theory^n_{t}| 
	- \log|\theory^n_{t-1}|\Big)\\
	& = \sum_{t=1}^\infty 
	\big\{\log\text{-\# strings falsified when modifying }\theory_{t-1}\mapsto \theory_{t}\big\}.
\end{align*}
Thus, the hard falsifiability of $\vec{y}$ is the number of hypotheses Forecaster eliminates in the process of adapting its theory to the data. Note that theories are falsified \emph{prior} to predicting: at time $t$, Forecaster first eliminates hypotheses based on $y_{1:t}$ and then uses the new theory $\theory_{t+1}$ to predict $y_{t+1}$. 

\subsection{Falsifiable $\implies$ Learnable ($\uni$)}
\label{s:uni_l_is_f}

The main theorem for universal induction differs from statistical learning and sequential prediction, in that Forecaster's theory is not fixed. Falsifiability quantifies the hypotheses that Forecaster eliminates whilst adapting its theory. The more Forecaster is required to adapt -- \emph{prior} to predicting -- the weaker the guarantee on its predictive performance.

\begin{thm_uni}[main theorem, $\uni$]
	The predictive risk under Solomonoff induction (1) coincides with the expected error and (2) is bounded by the number of hypotheses Nature falsifies when choosing the string $\vec{y}$:
	\begin{equation}
		\error(\psi_\sol|\vec{y}) = \risk_{\psi_\sol}^\uni\big(\vec{y}\big) \leq \falsect_\turing^\uni(\vec{y}).
		\tag{E}
	\end{equation}		
\end{thm_uni}

\begin{proof}
	By Lemma~\ref{t:prisk_to_risk}, the predictive risk and risk coincide for universal induction: $\error^\uni(\psi|\vec{y}) = \risk_{\psi}^\uni\big(\vec{y}\big)$. 

	By Proposition~\ref{t:kolmogorov}, the hard falsifiability of $\vec{y}$ coincides with (the negative logarithm of) the Solomonoff prior
	\begin{equation*}
		\falsect^\uni_\turing(\vec{y}) = -\log\probq_\sol(\vec{y}).
	\end{equation*}
	Finally, the result follows by Solomonoff's Theorem~\ref{t:solomonoff}.
\end{proof}

More generally, Theorem~E suggests that Bayesian updating is a way of modifying theories, whose cost (measured in errors) can be bounded using falsifiability.

We conclude by relating falsifiability to Kolmogorov complexity. Intuitively, a string is simple if it is the output of a short computer program. More formally, 

\begin{defn}[Kolmogorov complexity]
	\label{d:kolmogorov}
	The \textbf{Kolmogorov complexity} of a string, with respect to Turing machine $\turing$, is the length of the shortest program that outputs the string as a prefix \cite{kolmogorov:65}:
	\begin{equation*}
		\kolmog_\turing(\vec{y}) := \min_{\vec{x}\in \X}\big\{\len(\vec{x}) \,\big|\, \turing(\vec{x}) = \vec{y}\bullet\big\}
	\end{equation*}
\end{defn}

The Kolmogorov complexity $\kolmog_\turing$ depends on the choice of Turing machine up to an additive constant that does not depend on $\vec{y}$ \cite{li:2008}.

\begin{prop}[relation between falsifiability and Kolmogorov complexity]
	Falsifiability lower bounds Kolmogorov complexity:
	\begin{equation*}
		\falsect_\turing^\uni(\vec{y}) \leq \kolmog_\turing(\vec{y}).
	\end{equation*}
	Further, $\falsect^\uni_\turing(\vec{y}) = \kolmog_\turing(\vec{y})$ up to an additive constant that does not depend on $\vec{y}$.
\end{prop}

\begin{proof}
	The inequality follows from the definitions of the Solomonoff prior and Kolmogorov complexity. 

	By Levin's coding theorem \cite{li:2008}, the Kolmogorov complexity of a string coincides with the negative log probability of the string according to the Solomonoff prior up to an additive constant.
\end{proof}

\subsection{Proofs ($\uni$)}
\label{s:uni_proofs}

Equip $\hypotheses^n$ with the uniform distribution and let $\prob_{\cS^n}(\X)$ denote the $\cS^n$-induced  distribution on $\X$. Recall that we defined the Solomonoff prior as the limit of the $\turing\circ \cS^n$-induced distribution on $\Y$ 
\begin{equation*}
	\probq_\sol(\vec{y}) := \lim_{n\rightarrow\infty} \prob_{\turing\circ\cS^n}(\vec{y}),		
\end{equation*}
where $\turing\circ\cS^n:\hypotheses^n\xrightarrow{\cS^n}\X\cup\{\emptyset\}\xrightarrow{\turing}\Y\cup\{\emptyset\}$.

\vspace{5mm}
\noindent
\textbf{Definition-Proposition~\ref{t:limit}.}
\emph{
	The following hold:
	\begin{enumerate}
		\item The limit $\prob_\cS(\X) := \lim_{n\rightarrow \infty}\prob_{n}(\X)$ is well-defined with
		\begin{equation*}
			\prob_\cS(\vec{x}) = 2^{-\len(\vec{x})}.
		\end{equation*}
		\item The limit $\probq_{\sol}(\Y):=\prob_{\turing \circ \cS}(\Y) = \lim_{n\rightarrow \infty}\prob_{\turing\circ \cS^n}(\Y)$ is well-defined and coincides with the Solomonoff prior. That is,
		\begin{equation*}
			\probq_\sol(\vec{y}) =  \sum_{\{\vec{x}\in\X|\turing(\vec{x}) = \vec{y}\bullet\}} 2^{-\len(\vec{x})}.
		\end{equation*}
	\end{enumerate}
}
\begin{proof}
	\begin{enumerate}[\emph{Claim} 1.]
		\item By Lemma~\ref{t:finite_det}, the induced probability of a valid program is 
		\begin{equation*}
			\prob_{\cS^n}(\vec{x}) = \begin{cases}
				\sum_{\vec{x}\vec{s}} \frac{1}{2^n} 
				= \frac{2^{n-\len(\vec{x})}}{2^n} = 2^{-\len(\vec{x})}
			& \text{if }\len(\vec{x})\leq n\\
			0 & \text{else.}
			\end{cases}		
		\end{equation*}
		Thus, $\lim_{n\rightarrow\infty} \prob_{\cS^n}(\vec{x}) = 2^{-\len(\vec{x})}$ for all valid programs.
		\item Also by Lemma~\ref{t:finite_det}.
	\end{enumerate}	
\end{proof}

Recall that the standard definition of Solomonoff induction is as the strategy:
\begin{equation*}
	(\psi_\sol)_t:Y^{t-1}\rightarrow \Delta_Y:y_{1:t-1}
	\mapsto \probq_\sol(y_t|y_{1:t-1})
	:= \frac{\probq_\sol(y_{1:t})}{\probq_\sol(y_{1:t-1})}.
\end{equation*}

\vspace{5mm}
\noindent
\textbf{Definition-Proposition~\ref{t:sol_strategy}.}
\emph{	
	The two definitions of Solomonoff induction coincide:
	\begin{equation*}
		\lim_{n\rightarrow \infty} \probq_{n,t}(y_t) 
		= \frac{\probq_\sol(y_{1:t})}{\probq_\sol(y_{1:t-1})}.
	\end{equation*}
}

\begin{proof}
	The theory $\theory^n_t$ is the set of all strings of length $\leq n$ consistent with the observation $y_{1:t-1}$. Pushing the uniform distribution on $\hypotheses^n$ forward onto $Y^\infty$ yields, asymptotically, the conditional Solomonoff distribution.
\end{proof}

\begin{lem}[predictive risk reduces to risk]\label{t:prisk_to_risk}
	If $\vec{y}$ is computable then
	\begin{equation*}
		\error^\uni(\psi|\vec{y}) 
		:= \lim_{n\rightarrow \infty}\error^\uni_\psi(\theory^n|\vec{y}) 
		= \risk_\psi^\uni\big(\vec{y}\big).
	\end{equation*}
\end{lem}

\begin{proof}
	As $n\rightarrow\infty$, the theory incorporates all valid programs, and so can match any computable sequence. Thus,
	\begin{equation*}
		\lim_{n\rightarrow \infty}\risk^\uni_{\theory^n}(\vec{y}) = 0
	\end{equation*}
	and the result follows.
\end{proof}

\begin{prop}[hard falsifiability and Solomonoff prior]
	\label{t:kolmogorov}
	The hard falsifiability of string $\vec{y}$ for Turing machine $\turing$ is
	\begin{equation*}
		\falsect_\turing^\uni(\vec{y}) 
		= -\log \probq_\sol(\vec{y}).
	\end{equation*}
\end{prop}

\begin{proof}
	Observe that the risk factorizes as
	\begin{equation*}
		\begin{matrix}
			\risk^\uni_{\vec{y}}: & \X & \xrightarrow{\turing} & \Y & \xrightarrow{\sum\loss} & \bR \\
			& \vec{x} & \mapsto & \turing(\vec{x}) & \mapsto & \sum_{t=1}^\infty \loss\big(\turing(\vec{x})_t,y_t\big).
		\end{matrix}		
	\end{equation*}
	The proposition follows from the following two claims.
	
	\begin{enumerate}[\emph{Claim} 1.]
		\item \emph{$\information(\turing, \vec{y}) = -\log \probq_\sol(\vec{y})$ for all $\vec{y}\in \Y$.}

		Consider the function $\turing:\X\rightarrow Y$, where $\X$ is equipped with the distribution $\prob_\cS(\X)$ from Proposition~\ref{t:limit}. Since Turing machines are deterministic, we have that $\prob_\turing(\vec{y}|\vec{x})=1$, and so
		\begin{equation*}
			\prob_\turing(\vec{x}|\vec{y}) 
			= \prob_\turing(\vec{y}|\vec{x}) \cdot \frac{\prob_{\cS}(\vec{x})}{\prob_\turing(\vec{y})} 
			= \frac{\prob_{\cS}(\vec{x})}{\prob_\turing(\vec{y})} 
		\end{equation*}
		It follows that
		\begin{align*}
			\information(\turing, \vec{y})
			& = \kl\Big[\prob_\turing(\X|\vec{y})\,\Big\|\, \prob_{\cS}(\X)\Big] 
			= \sum_{\vec{x}\in \X} \prob_\turing(\vec{x}|\vec{y}) \log \frac{\prob_\turing(\vec{x}|\vec{y})}{\prob_{\cS}(\vec{x})} \\
			& = \sum_{\vec{x}\in \X} \prob_\turing(\vec{x}|\vec{y}) \log \frac{\prob_\turing(\vec{y}|\vec{x})\cdot \prob_{\cS}(\vec{x})}{\prob_\turing(\vec{y})\cdot \prob_{\cS}(\vec{x})} 
			= \sum_{\vec{x}\in \X} \prob_\turing(\vec{x}|\vec{y}) \log \frac{1}{\prob_\turing(\vec{y})} \\
			& = - \log \prob_\turing(\vec{y}) \\
			& = - \log \probq_\sol(\vec{y}).
		\end{align*}
		where the last equality follows from Proposition~\ref{t:limit}.

		\item $\falsect^\uni(\vec{y}) = \information(\turing,\vec{y})$.

		Follows from $\falsect^\uni(\vec{y}) = \information(\risk_{\vec{y}},0)$ and $\risk_{\vec{y}}^{-1}(0)=\turing^{-1}(\vec{y})$.
	\end{enumerate}	
	Concatenating the claims yields the desired result.
\end{proof}

\begin{thm}[generalization bound for Solomonoff induction]\label{t:solomonoff}
	\begin{equation*}
		\sum_{t=1}^\infty \expec \loss\Big(\psi_\sol(y_{1:t-1}),y_t\Big) 
		\leq -\log \probq_\sol(\vec{y}).
	\end{equation*}		
\end{thm}

\begin{proof}
	The following proof is taken from \cite{hutter:11}:
	\begin{align*}
		\sum_{t=1}^\infty \expec \loss\Big(\psi_\sol(y_{1:t-1}),y_t\Big)
		& = \sum_{t=1}^\infty \big|1 - \probq_\sol(y_t | y_{1:t-1})\big|\\
		& \leq - \sum_{t=1}^\infty \log \probq_\sol(y_t|y_{1:t-1})\\
		& = -\log \probq_\sol(\vec{y}),
	\end{align*}
	where the inequality holds because $1-x\leq -\log x$.
\end{proof}

\subsection{Interpreting Solomonoff induction as a universal theory}
\label{s:uni_std}

Under the standard interpretation, Forecaster's theory is $\theory$ and $\falsect^\uni_\turing(\vec{y})$ counts the hypotheses falsified by Nature whilst choosing $\vec{y}$:
\begin{align*}
	 \falsect_\turing^{\uni}(\vec{y}) 
	 & = \lim_{n\rightarrow \infty}
	 \Big[ \log\big\{\textrm{\# strings of length $n$}\big\}
	- \log\big\{\textrm{\# that output $y$}\big\}\Big] \\
	& = \lim_{n\rightarrow \infty} \Big\{\textrm{$\log$-\# strings of length $n$ that Nature falsifies }\Big\}.
\end{align*}

\section{Discussion}
\label{s:discussion}

{\footnotesize
\begin{quote}
	\emph{[A] theory of induction is superfluous. It has no function in a logic of science. The best we can say of a hypothesis\footnote{This paper uses `theory' in the sense that Popper uses `hypothesis'.} is that up to now it has been able to show its worth, and that it has been more successful than other hypotheses although, in principle, it can never be justified, verified, or even shown to be probable. This appraisal of the hypothesis relies solely upon deductive consequences (predictions) which may be drawn from the hypothesis: There is no need even to mention `induction'.}

	\hfill -- from \cite{popper:59}. 
\end{quote}
}

We conclude by discussing the paper's implications for scientific inference, focusing on the ideas of Karl Popper. According to Popper, inductive inference is meaningless. As an alternative, he advocated \emph{hypothetico-deductive inference}, which proceeds as follows \cite{gelman:13}.

Forecaster makes observations, proposes a theory, and deduces consequences. A theory is scientific if it is \emph{falsifiable}. That is, if it is possible to deduce empirically testable consequences. The scientific method, according to Popper, is: to propose falsifiable theories that are in line with past observations; to subject them to severe empirical tests; and to discard and replace them if and when they are falsified. 

Popper's ideas are extremely influential in the scientific community. Indeed, he is essentially the only philosopher that scientists draw on as a resource to evaluate and compare theories. Philosophers, however, consider Popper's approach to be fundamentally flawed \cite{godfrey:11}. The three main problems that have been identified are:
\begin{enumerate}[P1.]
	\item \emph{Infinite alternatives.}
	The set of imaginable hypotheses is infinite, so that it is trivial to find a collection of specific hypotheses that a specific theory falsifies. 
	\item \emph{Stochasticity.}
	It is unclear how to apply Popper's ideas to stochastic theories, which cannot be definitely falsified. 
	\item \emph{No confirmation.}
	Popper rejected the notion that positive evidence should increase our confidence in a scientific theory. Rejecting confirmation eliminates any rationale, aside from habit, for using a well-tested theory over a brand new theory, assuming both are falsifiable. 
\end{enumerate}

Our formulation of falsifiability does not exactly line up with what Popper had in mind. We proceed regardless.

Problem \emph{P1} is solved by restricting attention to the finite set of effective hypotheses. Problem \emph{P2} is also solved as a corollary of our results. Soft and hard falsifiability are defined with respect to \emph{deterministic} hypotheses, whereas the predictive risk allows \emph{probabilistic} hypotheses.

Problem \emph{P3} is more interesting. If Nature is \emph{i.i.d.} then Theorem~D'' provides a guarantee on a predictor's future accuracy that depend on the theory's falsifiability and the predictor's past performance. Thus, with the addition of the \emph{i.i.d.} assumption, there \emph{is} quantifiable confirmation. 

If no assumptions are made about Nature's behavior, then the setting is sequential prediction. The most that can be said is that, if a theory is falsifiable, then its predictive performance can be as good as its explanatory performance in hindsight. Nothing \emph{absolute} can be said about predictive performance \emph{a priori}.

Finally, Solomonoff induction is purported to be a (non-computable) theory that optimally explains and predicts every computable string. However, observe that Theorem~E says \emph{nothing} about Solomonoff induction's predictive performance unless $\falsect^\uni_\turing(\vec{y})$ or the Kolmogorov complexity $\kolmog_\turing(\vec{y})$ are known \emph{a priori} -- which is never the case. For example, suppose Nature picks a string that contains $10^9$ zeros followed by $10^9$ coin flips, followed by only zeros. Solomonoff induction's error rate on the first billion instances will not be indicative of its performance on the next billion. Assuming that Nature chooses strings with low Kolmogorov complexity is analogous to, albeit weaker than, assuming Nature is \emph{i.i.d.}

The current state-of-the-art in learning theory therefore supports Popper's intuitions about falsifiability -- including his rejection of confirmation. In a more positive vein, learning theory suggests that inductive inference requires additional assumptions and provides tools for analyzing their implications.

\vspace{3mm}\noindent\textbf{Acknowledgments.} I am grateful to Samory Kpotufe, Jacob Abernethy and Pedro Ortega for useful discussions.


{\footnotesize
\begin{thebibliography}{00}

\ifx \showCODEN    \undefined \def \showCODEN     #1{\unskip}     \fi
\ifx \showDOI      \undefined \def \showDOI       #1{{\tt DOI:}\penalty0{#1}\ }
  \fi
\ifx \showISBNx    \undefined \def \showISBNx     #1{\unskip}     \fi
\ifx \showISBNxiii \undefined \def \showISBNxiii  #1{\unskip}     \fi
\ifx \showISSN     \undefined \def \showISSN      #1{\unskip}     \fi
\ifx \showLCCN     \undefined \def \showLCCN      #1{\unskip}     \fi
\ifx \shownote     \undefined \def \shownote      #1{#1}          \fi
\ifx \showarticletitle \undefined \def \showarticletitle #1{#1}   \fi
\ifx \showURL      \undefined \def \showURL       #1{#1}          \fi

\bibitem[\protect\citeauthoryear{Abernethy, Agarwal, Bartlett, and
  Rakhlin}{Abernethy et~al\mbox{.}}{2009}]%
        {abernethy:09}
{Jacob Abernethy}, {Alekh Agarwal}, {Peter~L Bartlett}, {and} {Alexander
  Rakhlin}. 2009.
\newblock \showarticletitle{A stochastic view of optimal regret through minimax
  duality}. In {\em COLT}.
\newblock


\bibitem[\protect\citeauthoryear{Balduzzi}{Balduzzi}{2011}]%
        {balduzzi:11ilf}
{David Balduzzi}. 2011.
\newblock \showarticletitle{Information, learning and falsification}, In
  Philosophy and {M}achine {L}earning workshop, Neural Information Processing
  Systems (NIPS). {\em ar{X}iv\/} (2011).
\newblock


\bibitem[\protect\citeauthoryear{Balduzzi}{Balduzzi}{2013}]%
        {balduzzi:11ffp}
{David Balduzzi}. 2013.
\newblock \showarticletitle{Falsification and {F}uture {P}erformance}.
\newblock In {\em Algorithmic {P}robability and {F}riends: {B}ayesian
  {P}rediction and {A}rtificial {I}ntelligence}, {David Dowe} (Ed.). LNAI, Vol.
  7070. Springer, 65--78.
\newblock


\bibitem[\protect\citeauthoryear{Boucheron, Lugosi, and Massart}{Boucheron
  et~al\mbox{.}}{2000}]%
        {boucheron:00}
{S Boucheron}, {G Lugosi}, {and} {P Massart}. 2000.
\newblock \showarticletitle{A {S}harp {C}oncentration {I}nequality with
  {A}pplications}.
\newblock {\em Random Structures and Algorithms\/} {16}, 3 (2000), 277--292.
\newblock


\bibitem[\protect\citeauthoryear{Bousquet, Boucheron, and Lugosi}{Bousquet
  et~al\mbox{.}}{2004}]%
        {bousquet:04}
{Olivier Bousquet}, {St{\'e}phane Boucheron}, {and} {G{\'a}bor Lugosi}. 2004.
\newblock \showarticletitle{Introduction to {S}tatistical {L}earning {T}heory}.
\newblock In {\em Advanced Lectures on Machine Learning}, {O~Bousquet}, {U~von
  Luxburg}, {and} {G~R{\"a}tsch} (Eds.). Springer, 169--207.
\newblock


\bibitem[\protect\citeauthoryear{Cesa-Bianchi and Lugosi}{Cesa-Bianchi and
  Lugosi}{2006}]%
        {cesa:06}
{Nicolo Cesa-Bianchi} {and} {Gabor Lugosi}. 2006.
\newblock {\em Prediction, {L}earning and {G}ames}.
\newblock Cambridge University Press.
\newblock


\bibitem[\protect\citeauthoryear{Corfield, Sch{\"o}lkopf, and Vapnik}{Corfield
  et~al\mbox{.}}{2009}]%
        {corfield:09}
{David Corfield}, {Bernhard Sch{\"o}lkopf}, {and} {V Vapnik}. 2009.
\newblock \showarticletitle{Falsification and {S}tatistical {L}earning
  {T}heory: {C}omparing the {P}opper and {V}apnik-{C}hervonenkis {D}imensions}.
\newblock {\em Journal for General Philosophy of Science\/} {40}, 1 (2009),
  51--58.
\newblock


\bibitem[\protect\citeauthoryear{Gelman and Shalizi}{Gelman and
  Shalizi}{2013}]%
        {gelman:13}
{Andrew Gelman} {and} {Cosma Shalizi}. 2013.
\newblock \showarticletitle{Philosophy and the practice of {B}ayesian
  statistics}.
\newblock {\it Brit. J. Math. Statist. Psych.}  {66} (2013), 8--38.
\newblock


\bibitem[\protect\citeauthoryear{Godfrey-Smith}{Godfrey-Smith}{2011}]%
        {godfrey:11}
{Peter Godfrey-Smith}. 2011.
\newblock \showarticletitle{Popper's {P}hilosophy of {S}cience: {L}ooking
  {A}head}. In {\em The {C}ambridge {C}ompanion to {P}opper}, {J~Shearmur}
  {and} {G~Stokes} (Eds.). Cambridge University Press.
\newblock


\bibitem[\protect\citeauthoryear{Harman and Kulkarni}{Harman and
  Kulkarni}{2007}]%
        {harman:07}
{Gilbert Harman} {and} {Sanjeev Kulkarni}. 2007.
\newblock {\em Reliable {R}easoning: {I}nduction and {L}earning {T}heory}.
\newblock MIT Press.
\newblock


\bibitem[\protect\citeauthoryear{Hutter}{Hutter}{2011}]%
        {hutter:11}
{Marcus Hutter}. 2011.
\newblock \showarticletitle{{U}niversal {L}earning {T}heory}.
\newblock In {\em {E}ncyclopedia of {M}achine {L}earning}, {Claude Sammut}
  {and} {Geoffrey~I Webb} (Eds.). Springer.
\newblock


\bibitem[\protect\citeauthoryear{Kolmogorov}{Kolmogorov}{1965}]%
        {kolmogorov:65}
{A~N Kolmogorov}. 1965.
\newblock \showarticletitle{Three approaches to the quantitative definition of
  information}.
\newblock {\em Problems Inform. Transmission\/} {1}, 1 (1965), 1--7.
\newblock


\bibitem[\protect\citeauthoryear{Koltchinskii}{Koltchinskii}{2001}]%
        {koltchinskii:01}
{V Koltchinskii}. 2001.
\newblock \showarticletitle{Rademacher penalties and structural risk
  minimization}.
\newblock {\em IEEE Trans. Inf. Theory\/}  {47} (2001), 1902--1914.
\newblock


\bibitem[\protect\citeauthoryear{Li and Vit{\'a}nyi}{Li and
  Vit{\'a}nyi}{2008}]%
        {li:2008}
{M Li} {and} {P Vit{\'a}nyi}. 2008.
\newblock {\em An {I}ntroduction to {K}olmogorov {C}omplexity and {I}ts
  {A}pplications}.
\newblock Springer.
\newblock


\bibitem[\protect\citeauthoryear{Popper}{Popper}{1959}]%
        {popper:59}
{Karl Popper}. 1959.
\newblock {\em The {L}ogic of {S}cientific {D}iscovery}.
\newblock Hutchinson.
\newblock


\bibitem[\protect\citeauthoryear{Rakhlin and Sridharan}{Rakhlin and
  Sridharan}{2014}]%
        {rakhlin:14}
{Alexander Rakhlin} {and} {Karthik Sridharan}. 2014.
\newblock S{T}{A}{T}928: Statistical {L}earning {T}heory and {S}equential
  {P}rediction.
\newblock Lecture Notes.
\newblock


\bibitem[\protect\citeauthoryear{Rakhlin, Sridharan, and Tewari}{Rakhlin
  et~al\mbox{.}}{2014}]%
        {rakhlin:14a}
{Alexander Rakhlin}, {Karthik Sridharan}, {and} {Ambuj Tewari}. 2014.
\newblock \showarticletitle{Online {L}earning via {S}equential {C}omplexities}.
  In {\em JMLR}.
\newblock


\bibitem[\protect\citeauthoryear{Solomonoff}{Solomonoff}{1964}]%
        {solomonoff:64}
{R~J Solomonoff}. 1964.
\newblock \showarticletitle{A formal theory of inductive inference {I},
  {I}{I}}.
\newblock {\em Inform. Control\/} {7}, 1-22, 224-254 (1964).
\newblock


\bibitem[\protect\citeauthoryear{Vapnik}{Vapnik}{1995}]%
        {vapnik:95}
{V Vapnik}. 1995.
\newblock {\em The {N}ature of {S}tatistical {L}earning {T}heory}.
\newblock Springer.
\newblock


\end{thebibliography}
}
\end{document}